\PassOptionsToPackage{dvipsnames}{xcolor}
\documentclass[final,12pt]{WLgo} % Include author names

\makeatletter
\renewcommand{\@jmlrpages}{} 

\renewcommand{\@jmlrvolume}{}
\renewcommand{\@jmlrissue}{}
\renewcommand{\@jmlryear}{}

\def\ps@jmlrtps{
  \let\@mkboth\@gobbletwo
  \def\@oddhead{}
  \def\@oddfoot{\@titlefoot} % 하단(Footer) 저작권 표시는 유지하려면 남김, 지우려면 {}로 비움
  \def\@evenhead{}
  \def\@evenfoot{\@oddfoot}
}

\def\ps@jmlrps{%
  \let\@mkboth\@gobbletwo
  \def\@oddhead{\hfill {\small\scshape \@shorttitle} \hfill}%
  \def\@evenhead{\hfill {\small\scshape \@shorttitle} \hfill}%
  
  \def\@oddfoot{\hfill \small\rmfamily \thepage \hfill}%
  \def\@evenfoot{\hfill \small\rmfamily \thepage \hfill}%
}

\makeatother

\usepackage{microtype}
\usepackage{subcaption}
\usepackage{amsmath}
\usepackage{booktabs} % for professional tables
\usepackage{hyperref}
\usepackage{mathtools}
\usepackage{enumitem}
\usepackage{bbold}
\usepackage{pgfplots}
\pgfplotsset{compat=1.18}
\usepackage{tikz}
\usepackage{tikz-cd}
\usetikzlibrary{arrows.meta,shapes,positioning,calc,fit,chains,shapes.geometric, decorations.text}
\usepackage[capitalize,noabbrev]{cleveref}
\usepackage{booktabs}  
\usepackage{multirow}   
\usepackage{graphicx}

\title[WL Go Categorical]{Weisfeiler and Lehman Go Categorical}
\coltauthor{%
 \Name{Seongjin Choi} \Email{jincs10@postech.ac.kr}\\
 \addr POSTECH, Gyeongbuk, Korea \& Center for Geometry and Physics, Institute for Basic Science (IBS), 79 Jigok-ro 127beon-gil, Nam-gu, Pohang, Gyeongbuk, KOREA 37673
 \AND
 \Name{Gahee Kim} \Email{gaheekim@kaist.ac.kr}\\
 \addr Kim Jaechul Graduate School of AI, Korea Advanced Institute of Science and Technology (KAIST), 85 Hoegi-ro, Dongdaemun-gu, Seoul, KOREA 02455
 \AND
\Name{Se-Young Yun} \Email{yunseyoung@kaist.ac.kr}\\
 \addr Kim Jaechul Graduate School of AI, Korea Advanced Institute of Science and Technology (KAIST), 85 Hoegi-ro, Dongdaemun-gu, Seoul, KOREA 02455
}

\begin{document}

\maketitle

\begin{abstract}%
While lifting map has significantly enhanced the expressivity of graph neural networks, extending this paradigm to hypergraphs remains fragmented. To address this, we introduce the categorical Weisfeiler-Lehman framework, which formalizes lifting as a functorial mapping from an arbitrary data category to the unifying category of graded posets. When applied to hypergraphs, this perspective allows us to systematically derive Hypergraph Isomorphism Networks, a family of neural architectures where the message passing topology is strictly determined by the choice of functor. We introduce two distinct functors from the category of hypergraphs: an incidence functor and a symmetric simplicial complex functor. While the incidence architecture structurally mirrors standard bipartite schemes, our functorial derivation enforces a richer information flow over the resulting poset, capturing complex intersection geometries often missed by existing methods. We theoretically characterize the expressivity of these models, proving that both the incidence-based and symmetric simplicial approaches subsume the expressive power of the standard Hypergraph Weisfeiler-Lehman test. Extensive experiments on real-world benchmarks validate these theoretical findings.
\end{abstract}


\section{Introduction}
\label{sec:introduction}

The expressivity of graph neural networks (GNNs) has long been bounded by the 1-Weisfeiler-Lehman (1-WL) test \citep{weisfeiler1968reduction, xu2018powerful, morris2019weisfeiler}. To transcend this limitation, recent research has gravitated toward higher-order graph learning, where graphs are lifted into richer domains---such as simplicial complexes \citep{bodnar2021weisfeiler} or regular cell complexes \citep{bodnar2021weisfeiler_cw, giusti2024topological}---to capture structures like cliques and cycles that are invisible to standard message passing \citep{arvind2020weisfeiler, chen2020can}. These methods have successfully unified geometric intuition, establishing a new standard for expressive graph learning \citep{eblisimplicial, wu2023simplicial, hajij2022simplicial}. 

However, extending this paradigm to hypergraphs \citep{bretto2013hypergraph} remains a significant open challenge. Unlike graphs, where cliques naturally lift to simplices, hypergraphs consist of arbitrary subsets of nodes with no canonical geometric realization. Consequently, current hypergraph neural networks \citep{feng2019hypergraph} often rely on graph-based reductions, such as the bipartite expansion \citep{feng2019hypergraph, huang2021unignn, feng2024hypergraph, zhangimproved}.

In this work, we propose a unified, rigorous framework to bridge this gap. We argue that the fragmentation of current methods stems from the lack of a common mathematical language to describe lifting. We introduce category theory \citep{mac1998categories} not as an abstraction, but as a practical engineering tool to solve the hypergraph lifting problem. We introduce the Categorical WL (CatWL) framework, a generalized isomorphism testing machinery that operates on a universal domain, graded poset \citep{stanley2011enumerative}. Our core insight is that while input data types (graphs, hypergraphs, simplicial complexes) vary, their structural essence can always be faithfully encoded as a graded poset. By defining lifting maps \citep{bodnar2021weisfeiler_cw} as functors from the category of data $\mathcal{C}$ to the category of graded posets, we provide a rigorous guarantee that the lifting process preserves the isomorphism class of the original object. This functorial perspective allows us to systematically derive neural architectures. We demonstrate that the choice of functor dictates the message passing topology, leading to a family of models we term Categorical Message Passing Networks (CatMPNs). Specifically, we investigate two distinct functors: The Incidence Poset Functor, which treats hyperedges as dimension 1 objects, recovering the standard bipartite expansion \citep{huang2021unignn, feng2024hypergraph, zhangimproved} and the symmetric simplicial complex \citep{choi2025hypergraph} functor, which treats hyperedges as fully structured simplices, and a cardinality-aware message passing scheme. 
Our contributions are summarized as follows: \\
- Categorical Framework: We propose the CatWL framework, which unifies lifting strategies \citep{bodnar2021weisfeiler, bodnar2021weisfeiler_cw} under the lens of category theory.\\
- Universal Domain: We identify graded posets as the optimal universal domain for WL tests, capable of subsuming hypergraphs \citep{xu2018powerful, feng2024hypergraph}, simplicial complexes \citep{bodnar2021weisfeiler}, and regular cell complexes \citep{bodnar2021weisfeiler_cw} within a single computational structure. \\
- Neural Architectures: We derive CatMPNs as the neural network version of our CatWL test. We theoretically characterize the expressivity of two specific instances. \\
- Empirical Validation: We provide extensive experiments verifying that CatMPN, in the category of hypergraphs, outperforms baseline methods on real-world benchmarks.

\section{Why Category Theory}
\label{sec:why_category}

Recent advancements in GNNs have largely been driven by the pursuit of higher expressivity beyond the 1-WL limitation \citep{morris2019weisfeiler, bodnar2021weisfeiler, bodnar2021weisfeiler_cw}. While this has led to powerful lifting map strategies for graphs \citep{bodnar2021weisfeiler, bodnar2021weisfeiler_cw}, extending these methods to hypergraphs remains fragmented due to structural incompatibilities.

\paragraph{The Lifting Mismatch} In this section, we identify the limitations of current graph-based reductions for higher-order data and propose a unified framework. We argue that graded posets serve as the optimal universal domain for lifting, and that category theory---specifically, functoriality---provides the necessary rigorous condition to ensure these lifts preserve isomorphism. In the domain of graphs, the strategy for breaking the 1-WL barrier is well-established: we upgrade standard GNNs by imposing new geometric or algebraic constructions on the graph. Approaches such as Simplicial WL (SWL) and Cellular WL (CWL) lift the input graph into higher dimensions—for instance, by constructing clique complexes or effectively gluing 2-dimensional discs to fill graph cycles. These constructions unlock a richer set of interactions, allowing models to explicitly process topological features like boundary and co-boundary adjacencies. However, these specific geometric methods are not applicable to hypergraphs. Unlike graphs, where the clique complex or cycle-filling offers a canonical topological enrichment, a hypergraph consists of arbitrary, irregular subsets of nodes. There exists no equivalent canonical gluing procedure for hypergraphs. Consequently, current attempts to apply WL-like procedures often rely on the bipartite expansion, which treats hyperedges merely as a secondary node type. This flattens the higher-order structure, obscuring the intrinsic containment and intersection geometries critical for higher-order expressivity.

\paragraph{A Universal domain} A recurring strategy for surpassing standard WL is to lift a graph into a richer higher-order domain and perform refinement there. However, this landscape is currently fragmented: SWL targets simplicial complexes (via clique complexes), while CWL targets regular cell complexes (via attaching a 2-dimensional disc to each cycle in the graph). Because these methods map to distinct mathematical spaces, their resulting notions of neighborhoods, updates, and expressivity are difficult to compare within a single theoretical framework. This difficulty is even more pronounced for hypergraphs. To address this, we propose a unified domain. We establish the space of graded posets as a universal target capable of encoding the incidence structure of hypergraphs as well as the face relations of simplicial and regular cell complexes. By fixing the domain to graded posets, we ensure that the coloring update rule and its neural relaxation can be defined once and reused across input modalities, rather than effectively redesigned for every new lifting construction.

\paragraph{Functorial Lifting} With graded posets established as the unified domain, the remaining challenge is defining the lifting map itself. We argue that treating the data domain merely as a set of objects is insufficient. To rigorously capture the structural invariants required for isomorphism testing, we must view the domain as a category. Consequently, the notion of a lifting map is naturally upgraded to that of a functor. This perspective provides the necessary validity condition: functoriality. While heuristics like the clique complex implicitly satisfy this, functoriality explicitly constrains the design space for hypergraphs. It mandates that a lifting strategy must map morphisms of hypergraphs to morphisms of induced graded posets. This requirement guarantees that the lifting respects the symmetries of the input data, serving as a strict guide for designing canonical, isomorphism-preserving representations in the poset domain.

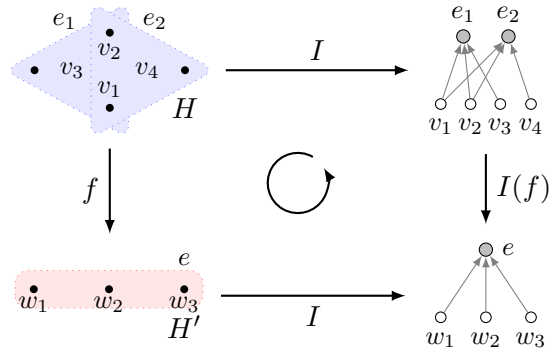
\begin{figure}[t]
    \centering
    
    \begin{tikzpicture}
        % Define styles for hypergraphs
        \tikzstyle{hg_node}=[circle, fill=black, inner sep=0pt, minimum size=3pt]
        \tikzstyle{hg_edge}=[draw=blue!50, fill=blue!10, rounded corners=5pt, ellipse, dotted, inner sep=10pt]
        
        % Define styles for posets (Hasse diagram)
        \tikzstyle{p_rank0}=[circle, draw=black, fill=white, inner sep=0pt, minimum size=4pt]
        \tikzstyle{p_rank1}=[circle, draw=black, fill=gray!50, inner sep=0pt, minimum size=5pt]
        \tikzstyle{p_line}=[thin, gray]

        % --- LEFT SIDE: Source Category (Hypergraphs) ---
        
        % Hypergraph X (Top Left)
        \node (X) at (0, 3) {
            \begin{tikzpicture}[scale=0.5]
                % Hyperedge e1 = {v1, v2, v3}
                \fill[hg_edge] (0.5,1.8) -- (-2.8,0) -- (0.5,-1.8) -- cycle;
                \fill[hg_edge] (-0.5,1.8) -- (2.8,0) -- (-0.5,-1.8) -- cycle; 
                \node[hg_node] (v1) at (0,-1) {};
                \node[hg_node] (v2) at (0,1) {};
                \node[hg_node] (v3) at (-2,0) {};
                \node[hg_node] (v4) at (2,0) {};
                
                \node at (-1.2, 1.3) {\small $e_1$}; 
                \node at (1.2, 1.3) {\small $e_2$}; 
                \node at (0, -0.5) {\small $v_1$};
                \node at (0, +0.5) {\small $v_2$};
                \node at (-1, 0) {\small $v_3$};
                \node at (1, 0) {\small $v_4$};
                \node at (2, -1) {$H$};
                
            \end{tikzpicture}
        };

        % Hypergraph Y (Bottom Left)
        % Y is a "collapsed" version (homomorphism target)
        \node (Y) at (0, 0) {
            \begin{tikzpicture}[scale=0.5]
                 % Hyperedge e1' = {u1, u2} (v2 and v3 collapsed to u2)
                \draw[hg_edge, fill=red!10, draw=red!50] (-2.5,-0.5) rectangle (2.5, 0.5);
                \node[hg_node] (u1) at (-2,0) {};
                \node[hg_node] (u2) at (2,0) {};
                \node[hg_node] (u3) at (0,0) {};
                \node at (2,0.8) {\small $e$};
                \node at (-2,-0.4) {\small $w_1$};
                \node at (0,-0.4) {\small $w_2$};
                \node at (2,-0.4) {\small $w_3$};
                \node at (2, -1) {$H'$};
            \end{tikzpicture}
        };

        % --- RIGHT SIDE: Target Category (Graded Posets) ---

        % Poset F(X) (Top Right)
        \node (FX) at (5, 3) {
            \begin{tikzpicture}[scale=0.6]
                % Rank 1 (Hyperedge)
                \node[p_rank1] (e1) at (0.5, 1.5) {};
                \node[p_rank1] (e2) at (1.5, 1.5) {};
                % Rank 0 (Nodes)
                \node[p_rank0] (pv1) at (0, 0) {};
                \node[p_rank0] (pv2) at (0.66, 0) {};
                \node[p_rank0] (pv3) at (1.32, 0) {};
                \node[p_rank0] (pv4) at (2, 0) {};
                % Incidence lines
                \draw[-latex, p_line] (pv1) -- (e1);
                \draw[-latex, p_line] (pv2) -- (e1);
                \draw[-latex, p_line] (pv3) -- (e1);
                \draw[-latex, p_line] (pv1) -- (e2);
                \draw[-latex, p_line] (pv2) -- (e2);
                \draw[-latex, p_line] (pv4) -- (e2);

                \node at (0.5, 2) {\small $e_1$}; 
                \node at (1.5, 2) {\small $e_2$}; 
                \node at (0, -0.5) {\small $v_1$};
                \node at (0.66, -0.5) {\small $v_2$};
                \node at (1.32, -0.5) {\small $v_3$};
                \node at (2, -0.5) {\small $v_4$};
            \end{tikzpicture}
        };

        % Poset F(Y) (Bottom Right)
        \node (FY) at (5, 0) {
            \begin{tikzpicture}[scale=0.6]
                % Rank 1 (Hyperedge)
                \node[p_rank1] (e) at (1, 1.5) {};
                % Rank 0 (Nodes)
                \node[p_rank0] (pu1) at (0, 0) {};
                \node[p_rank0] (pu2) at (1, 0) {};
                \node[p_rank0] (pu3) at (2, 0) {};
                % Incidence lines
                \draw[-latex, p_line] (pu1) -- (e);
                \draw[-latex, p_line] (pu2) -- (e);
                \draw[-latex, p_line] (pu3) -- (e);
                \node at (1.5,1.5) {\small $e$};
                \node at (0,-0.5) {\small $w_1$};
                \node at (1,-0.5) {\small $w_2$};
                \node at (2,-0.5) {\small $w_3$};
            \end{tikzpicture}
        };

        % --- ARROWS (Morphisms & Functors) ---

        % Morphism f: X -> Y
        \draw[-latex, thick] (X) -- node[left] {$f$} (Y);

        % Functor F (Top)
        \draw[-latex, thick] (X) -- node[above] {$I$} (FX);

        % Functor F (Bottom)
        \draw[-latex, thick] (Y) -- node[below] {$I$} (FY);

        % Morphism F(f): F(X) -> F(Y)
        \draw[-latex, thick] (FX) -- node[right] {$I(f)$} (FY);

               % --- Counter-Clockwise Commutativity Arrow ---
        % Center (2.5, 1.5), Radius 0.4
        % Starts at 60 degrees (top right), arcs CCW to -240 (top left)
        \draw[-latex, thick] (2.5, 1.5) ++(60:0.4) arc (60:390:0.4);
    \end{tikzpicture}
    
    \caption{$I$ is a functor from the category of hypergraphs to the category of graded posets by sending a hypergraph to its incidence poset. $I$ send a morphism of hypergraphs $f$ to a morphism of graded posets $I(f)$ and $I$ preserves the identity morphism. This implies $I$ is a lifting map.}
    \label{fig:functor_commutes}
\end{figure}

\section{Categorical WL test}
Having established the need for a unified domain and a functorial lifting mechanism, we now formally define the CatWL framework. In this section, we introduce graded posets as our universal computing domains, formalize the categorical machinery required for lifting, and finally define the CatWL test as a functor-guided coloring refinement procedure.

\subsection{Graded WL Test}
\label{subsec:graded_posets}

We begin by defining our target domain. The choice of graded posets \citep{stanley2011enumerative} is not arbitrary; it generalizes several discrete higher-order representations used in graph learning. In particular, incidence posets of hypergraphs, simplicial complexes, and regular cell complexes all induce natural graded poset structures when equipped with suitable dimension functions. Graded posets provide a flexible common domain capable of representing all these objects.

\begin{definition}[Graded Poset]\label{def:graded_poset}A graded poset $P$ \citep{stanley2011enumerative} is a poset $(P,<)$ with dimension function $\mathrm{dim}: P \to \mathbb{N}$ satisfying 
\begin{enumerate}[leftmargin=*,topsep=0pt,itemsep=-0.5ex]
    \item $p \leq p'$ implies $\mathrm{dim}(p) \leq \mathrm{dim}(p')$
    \item if there is no $q$ such that $p < q < p'$, we denote $p \prec p'$. In this case, $\mathrm{dim}(p')-\mathrm{dim}(p)=1$. 
\end{enumerate}
\end{definition}

Intuitively, a graded poset describes a collection of elements equipped with a hierarchical containment relation and an associated notion of dimension. The partial order specifies which elements are contained in or refine others, while the dimension function assigns an integer level to each element. As detailed in Appendix~\ref{app:Formaldenofliftingfunctors}, incidence posets, simplicial complexes, and regular cell complexes naturally induce graded poset structures. Since a graded poset generalizes the incidence poset of a graph, which is used for the WL test \citep{weisfeiler1968reduction}, it is natural to extend the WL test for graded posets.

\begin{definition}[Graded Weisfeiler-Lehman (GWL)]
\label{def:gwl}Given a graded poset $P$, we define a coloring refinement algorithm as follows:
\begin{enumerate}[leftmargin=*,topsep=0pt,itemsep=-0.5ex]
    \item $c^P_0$ is constant coloring on $P$
    \item When $c^P_t$ is given, $c^P_{t+1}(\sigma)$ is updated by colors of 4-adjacencies of $\sigma$ at iteration $t$ using perfect HASH function
    \begin{equation}
            c^P_{t+1}(\sigma) = \mathrm{HASH}\Bigl(c^P_t(\sigma),c^P_t(\mathcal{B}(\sigma)), c^P_t(\mathcal{C}(\sigma)), c^P_t(\mathcal{N}_{\downarrow}(\sigma)), c^P_t(\mathcal{N}_{\uparrow}(\sigma)) \Bigl)    
    \end{equation}
    where the colors of the 4-adjacencies of $\sigma$ are 
    \begin{itemize}[leftmargin=*,topsep=0pt,itemsep=-0.5ex]
        \item $c^P_t(\mathcal{B}(\sigma)) = \{\!\!\{c^P_t(\tau) \mid \tau \prec \sigma \}\!\!\}$, colors of boundary adjacencies of $\sigma$ 
        \item $c^P_t(\mathcal{C}(\sigma)) = \{\!\!\{c^P_t(\tau) \mid \sigma \prec \tau \}\!\!\}$, colors of coboundary adjacencies of $\sigma$
        \item $c^P_t(\mathcal{N}_{\downarrow}(\sigma)) = \{\!\!\{(c^P_t(\sigma'), c^P_t(\tau)) \mid \tau \prec \sigma,\sigma', \sigma\neq \sigma' \}\!\!\}$, colors of lower adjacencies of $\sigma$
        \item $c^P_t(\mathcal{N}_{\uparrow}(\sigma)) = \{\!\!\{(c^P_t(\sigma'), c^P_t(\tau)) \mid \sigma,\sigma' \prec \tau, \sigma\neq \sigma' \}\!\!\}$, colors of upper adjacencies of $\sigma$
    \end{itemize}
We say two graded posets $P, P'$ are isomorphic by the GWL test if the coloring histograms of $c^P_t, c^{P'}_t$ are equal for any $t \in \mathbb{N}$.
\end{enumerate}
\end{definition}

\begin{figure}[t]
    \centering
    {%
    \begin{tikzpicture}
        % Define styles for hypergraphs
        \tikzstyle{hg_node}=[circle, fill=black, inner sep=0pt, minimum size=3pt]
        \tikzstyle{hg_edge}=[draw=blue!50, fill=blue!10, rounded corners=5pt, ellipse, dotted, inner sep=10pt]
        
        % Define styles for posets (Hasse diagram)
        \tikzstyle{p_rank0}=[circle, draw=black, fill=white, inner sep=0pt, minimum size=4pt]
        \tikzstyle{p_rank1}=[circle, draw=black, fill=gray!50, inner sep=0pt, minimum size=5pt]
        \tikzstyle{p_line}=[thin, gray]

        % Poset F(X) (Top Right)
\node (FX) at (5, 3) {
    \begin{tikzpicture}
        % --- 1. 다이어그램 노드 (수직 간격 축소: 0, 1.0, 2.0) ---
        % Dim 2 (최상단)
        \node[p_rank1, BlueViolet] (e0) at (1, 2) {}; 
        
        % Dim 1 (중간)
        \node[p_rank1, Apricot] (e1) at (0.5, 1) {};
        \node[p_rank1, YellowOrange] (e2) at (1.5, 1) {};
        
        % Dim 0 (바닥)
        \node[p_rank0, Orchid] (pv1) at (0, 0) {};
        \node[p_rank0, VioletRed] (pv2) at (1, 0) {};
        \node[p_rank0, purple] (pv3) at (2, 0) {};

        % Incidence lines
        \draw[-latex, p_line] (e1) -- (e0);
        \draw[-latex, p_line] (e2) -- (e0);
        \draw[-latex, p_line] (pv1) -- (e1);
        \draw[-latex, p_line] (pv2) -- (e1);
        \draw[-latex, p_line] (pv2) -- (e2);
        \draw[-latex, p_line] (pv3) -- (e2);

        % 라벨 위치 조정
        \node at (2, 1) {$\sigma$};
        \node at (1.5, 2) {$\sigma_1$};
        \node at (0, 1) {$\sigma_2$};
        \node at (-0.5, 0) {$\sigma_3$};
        \node at (0.7, 0) {$\sigma_4$};
        \node at (2.5, 0) {$\sigma_5$};

        % --- 2. 수식 뭉치 (상단 라인 정렬) ---
        % anchor를 north west로 잡고, y좌표를 p1(e0)의 높이인 2.0에 맞춤
        % 약간의 미세조정이 필요하면 2.2 정도로 높이세요.
        \node[align=left, anchor=north west] at (3, 2.2) {
            $c(\mathcal{B}(\sigma)) = \{\!\!\{ \tikz\draw[fill=VioletRed, draw=black] (0,0) circle (2pt);, \tikz\draw[fill=purple, draw=black] (0,0) circle (2pt); \}\!\!\}$ \\[1ex]
            $c(\mathcal{C}(\sigma)) = \{\!\!\{ \tikz\draw[fill=BlueViolet, draw=black] (0,0) circle (2pt); \}\!\!\}$ \\[1ex]
            $c(\mathcal{N}_{\downarrow}(\sigma)) = \{\!\!\{ (\tikz\draw[fill=Apricot, draw=black] (0,0) circle (2pt);, \tikz\draw[fill=VioletRed, draw=black] (0,0) circle (2pt);) \}\!\!\}$ \\[1ex]
            $c(\mathcal{N}_{\uparrow}(\sigma)) = \{\!\!\{ (\tikz\draw[fill=Apricot, draw=black] (0,0) circle (2pt);, \tikz\draw[fill=BlueViolet, draw=black] (0,0) circle (2pt);) \}\!\!\}$
        };
    \end{tikzpicture}
};
    \end{tikzpicture}
    }
    \caption{Visualization of a single GWL refinement step on a sample graded poset $P$. The Hasse diagram depicts the covering relations $\prec$, where arrows point from covered elements to covering elements (e.g., $\sigma_4 \prec \sigma$). Nodes are colored according to the coloring $c$, and the 4-adjacency multisets for the element $\sigma$ are explicitly listed.}
    \label{fig:GWL}
\end{figure}
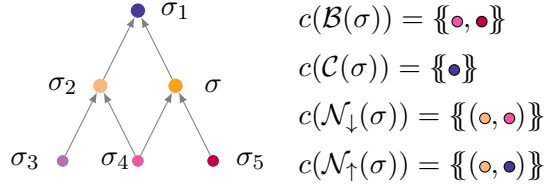

Figure~\ref{fig:GWL} illustrates the mechanism of the GWL coloring refinement. We consider a graded poset $P = {\sigma, \sigma_1, \dots, \sigma_5}$ with covering relations indicated by upward arrows. The specific containment structure includes $\sigma_2, \sigma \prec \sigma_1$; $\sigma_3, \sigma_4 \prec \sigma_2$; and $\sigma_4, \sigma_5 \prec \sigma$. The poset is initialized with a coloring $c$, assigning distinct colors to each element: $c(\sigma) = \tikz\draw[fill=YellowOrange, draw=black] (0,0) circle (2pt);$, $c(\sigma_1) = \tikz\draw[fill=BlueViolet, draw=black] (0,0) circle (2pt);$, $c(\sigma_2) = \tikz\draw[fill=Apricot, draw=black] (0,0) circle (2pt);$, $c(\sigma_3) = \tikz\draw[fill=Orchid, draw=black] (0,0) circle (2pt);$, $c(\sigma_4) = \tikz\draw[fill=VioletRed, draw=black] (0,0) circle (2pt);$, and $c(\sigma_5) = \tikz\draw[fill=purple, draw=black] (0,0) circle (2pt);$. During the refinement step, the color of $\sigma$ ($ \tikz\draw[fill=YellowOrange, draw=black] (0,0) circle (2pt); $) is updated by aggregating colors of 4-adjacencies of $\sigma$ in Figure~\ref{fig:GWL}.

% The boundary aggregation $c(\mathcal{B}(p))$ collects the multiset of colors from $p_4, p_5$, yielding $\{\!\!\{ \tikz\draw[fill=VioletRed, draw=black] (0,0) circle (2pt);, \tikz\draw[fill=purple, draw=black] (0,0) circle (2pt); \}\!\!\}$. More notably, the lower adjacency aggregation $c(\mathcal{N}_{\downarrow}(p))$ captures lateral interactions mediated by shared faces. Since $p$ shares the face $p_4$ ($ \tikz\draw[fill=VioletRed, draw=black] (0,0) circle (2pt); $) with its neighbor $p_2$ ($ \tikz\draw[fill=Apricot, draw=black] (0,0) circle (2pt); $), the aggregation collects the tuple of their colors:
% $c(\mathcal{N}_{\downarrow}(p)) = \left\{\!\!\left\{ \left( \tikz\draw[fill=Apricot, draw=black] (0,0) circle (2pt);, \tikz\draw[fill=VioletRed, draw=black] (0,0) circle (2pt); \right) \right\}\!\!\right\}$.
% Note that since

\subsection{Categories and Functors}
\label{subsec:cat_functor}

To apply the GWL machinery to arbitrary data types like hypergraphs, we must formally connect data to a graded poset. We achieve this via category theory.

\begin{definition}[Category]
\label{def:category}A category $\mathcal{C}$ \citep{mac1998categories} consists of a class of objects $\mathcal{C}$ and a class of morphisms $\text{Hom}(\mathcal{C})$. For every pair of objects $X, Y \in \mathcal{C}$, there exists a set of morphisms $\mathcal{C}(X, Y)$, denoting all morphisms from $X$ to $Y$. The category must satisfy two axioms:
\begin{enumerate}[leftmargin=*,topsep=0pt,itemsep=-0.5ex]
    \item Composition: For any morphisms $f \in \mathcal{C}(X, Y)$ and $g \in \mathcal{C}(Y, Z)$, there exists a composite morphism $g \circ f \in \mathcal{C}(X, Z)$ that is associative.
    \item Identity: For every object $X$, there exists an identity morphism $\text{Id}_X \in \mathcal{C}(X, X)$ such that $f \circ \text{Id}_X = f$ and $\text{Id}_Y \circ f = f$ for any $f \in \mathcal{C}(X,Y)$.
\end{enumerate}
We say two objects $X,Y$ isomorphic if there exists $f \in \mathcal{C}(X,Y), g \in \mathcal{C}(Y,X)$ satisfying $g \circ f = \text{Id}_X, f \circ g = \text{Id}_Y$.
\end{definition}

Common examples include the category of graphs ($\mathbf{Gph}$), the category of hypergraphs ($\mathbf{Hyp}$), the categories of simplicial and regular cell complexes, and the category of graded posets ($\mathbf{Poset}$) (formal definitions and proofs are deferred to Appendix~\ref{app:Formaldenofliftingfunctors}).

\begin{definition}[Functor]
\label{def:functor}Let $\mathcal{C}$ and $\mathcal{D}$ be categories. A functor $F: \mathcal{C} \to \mathcal{D}$ \citep{mac1998categories} is a mapping that assigns to each object $X \in \mathcal{C}$ an object $F(X) \in \mathcal{D}$, and to each morphism $f \in \mathcal{C}(X,Y)$ in $\mathcal{C}$ a morphism $F(f) \in \mathcal{D}(F(X), F(Y))$ in $\mathcal{D}$. To be a valid functor, $F$ must preserve the categorical structure:
\begin{enumerate}[leftmargin=*,topsep=0pt,itemsep=-0.5ex]
    \item Preservation of Composition: $F(g \circ f) = F(g) \circ F(f)$ for all composable morphisms $f, g$.
    \item Preservation of Identity: $F(\text{id}_X) = \text{id}_{F(X)}$ for every object $X \in \mathcal{C}$.
\end{enumerate}
\end{definition}

Functor generalizes topological liftings since if $X \cong Y$ are isomorphic in $\mathcal{C}$, then $F(X) \cong F(Y)$ should be isomorphic in $\mathcal{D}$ by preservation of composition and identity. In our framework, we specifically focus on functors where the target category is $\mathcal{D} = \mathbf{Poset}$. A functor $F: \mathcal{C} \to \mathbf{Poset}$ represents a systematic transformation of a data class $\mathcal{C}$ into graded posets. 

\subsection{Categorical WL Test}
\label{subsec:cat_wl_test}

Given a category $\mathcal{C}$, determining whether two objects in $\mathcal{C}$ are isomorphic or not is a hard task. We solve this task by lifting objects in $\mathcal{C}$ to graded posets via a functor $F: \mathcal{C} \to \mathbf{Poset}$. With the functor $F$, we can lift GWL on $\mathbf{Poset}$ to $\mathcal{C}$ to determine whether two objects in $\mathcal{C}$ are isomorphic. This is the $F$-CatWL test we will describe.

\begin{definition}[CatWL]
\label{def:cat_wl}
Let $F: \mathcal{C} \to \mathbf{Poset}$ be a functor. For any object $X \in \mathcal{C}$, the $F$-Categorical WL ($F$-CatWL) proceeds in two steps:
\begin{enumerate}[leftmargin=*,topsep=0pt,itemsep=-0.5ex]
    \item Associate an object $X$ in $\mathcal{C}$ to its graded poset $F(X)$ via the functor $F$.
    \item Apply GWL (Definition~\ref{def:gwl}) to $F(X)$.
\end{enumerate}
Two objects $X, X' \in \mathcal{C}$ are said to be isomorphic by \textbf{$F$-CatWL test}, denoted $X \cong_{F} X'$, if and only if $F(X), F(X')$ are isomorphic by GWL.
\end{definition}

Since any functor $F: \mathcal{C} \to \mathbf{Poset}$ preserves isomorphism structure, we obtain the implication chain $X \cong X' \implies F(X) \cong F(X') \implies F(X) \cong_{\text{GWL}} F(X')$. This last equivalence defines our $F$-CatWL distinguishability, denoted $X \cong_F X'$. Therefore, the $F$-CatWL test provides a necessary algorithmic condition for isomorphism in $\mathcal{C}$. By anchoring the refinement in $F$, we effectively pull back the rigorous adjacency structure of the graded poset domain to facilitate isomorphism testing in the source category.

\section{Hypergraph Isomorphism Networks}
\label{sec:hypergraph_isomorphism_networks}

We have defined the $F$-CatWL, which relies on a discrete color refinement procedure on the lifted graded poset by the functor $F$. While theoretically robust, the use of $\mathrm{HASH}$ is computationally intractable for continuous feature learning. In this section, we relax these constraints to define $F$-Categorical Message Passing Networks, neural architectures that correspond to the $F$-CatWL using continuous vector representations. We derive distinct hypergraph neural network architectures---collectively termed $F$-Hypergraph Isomorphism Networks ($F$-HIN)---each with unique inductive biases and expressive capabilities.

\subsection{Categorical Message Passing Networks}

The transition from WL tests to neural networks follows the standard paradigm of geometric deep learning: replacing discrete colors with Euclidean feature vectors and replacing $\mathrm{HASH}$ with general message functions, aggregators.

\begin{definition}[$F$-CatMPN]\label{def:f_catmpn}Let $F: \mathcal{C} \to \mathbf{Poset}$ be a functor. For any object $X \in \mathcal{C}$, the $F$-Categorical Message Passing Network ($F$-CatMPN) operates as follows:
\begin{enumerate}[leftmargin=*,topsep=0pt,itemsep=-0.5ex]
    \item Associate an object $X$ in $\mathcal{C}$ to its graded poset $F(X)$ via the functor $F$.
    \item For each element $\sigma \in F(X)$, maintain a feature vector $h^X_t(\sigma) \in \mathbb{R}^d$ at iteration $t$. Compute four types of messages corresponding to the features of 4-adjacencies:
    \begin{align}
        m^X_{\mathcal{B},t}(\sigma) &:= \mathrm{AGG}_{\tau \prec \sigma} \left( M_{\mathcal{B}}\left(h^X_t(\sigma), h^X_t(\tau)\right) \right) \label{eq:msg_boundary} \\
        m^X_{\mathcal{C},t}(\sigma) &:= \mathrm{AGG}_{\sigma \prec \tau} \left( M_{\mathcal{C}}\left(h^X_t(\sigma), h^X_t(\tau)\right) \right) \label{eq:msg_coboundary} \\
        m^X_{\mathcal{N}_{\downarrow},t}(\sigma) &:= \mathrm{AGG}_{\tau \prec \sigma, \sigma'} \left( M_{\mathcal{N}_{\downarrow}}\left(h^X_t(\sigma), h^X_t(\sigma'), h^X_t(\tau)\right) \right) \label{eq:msg_lower} \\
        m^X_{\mathcal{N}_{\uparrow},t}(\sigma) &:= \mathrm{AGG}_{\sigma, \sigma' \prec \tau} \left( M_{\mathcal{N}_{\uparrow}}\left(h^X_t(\sigma), h^X_t(\sigma'), h^X_t(\tau)\right) \right) \label{eq:msg_upper}
    \end{align}
\end{enumerate}
Here, $M_{\mathcal{B}}, M_{\mathcal{C}}, M_{\mathcal{N}_{\downarrow}}, M_{\mathcal{N}_{\uparrow}}$ are learnable message functions, $\mathrm{AGG}$ is an aggregation function, and $U$ is a learnable update function.
\end{definition}

The relationship between this neural approximation and the theoretical test is formalized by the following theorem, which extends the standard WL-MPN equivalence \citep{xu2018powerful, bodnar2021weisfeiler, bodnar2021weisfeiler_cw} to the categorical setting.

\begin{theorem}[Expressivity of $F$-CatMPN]
\label{thm:expressivity}
Let $F: \mathcal{C} \to \mathbf{Poset}$ be a functor. The $F$-CatWL is at least as powerful as the $F$-CatMPN. If the update and aggregation functions are injective and the network has sufficient depths and widths, the $F$-CatMPN has the same expressive power as $F$-CatWL.
\end{theorem}

In particular, for a functor $F: \mathbf{Hyp} \to \mathbf{Poset}$, we define the $F$-Hypergraph Isomorphism Network ($F$-HIN) as the specific instantiation of $F$-CatMPN utilizing the message, aggregation, and update functions detailed in Appendix~\ref{app:FHINarchitecture}.

\subsection{Two Architectures from Two Functors}
\label{subsec:two_functors}

The power of our categorical framework lies in its modularity: the message passing architecture is determined entirely by the choice of the lifting functor $F$. Distinct functors induce distinct graded poset structures, thereby generating topologically unique message passing schemes. In this section, we introduce two primary functors---the incidence poset functor ($I$) and the symmetric simplicial complex functor ($S$)---and analyze the specific neural architectures they induce.

\begin{figure}[t]
    \centering
    % Define styles globally
    \tikzset{
        hg_node/.style={circle, fill=black, inner sep=0pt, minimum size=3pt},
        hg_edge/.style={draw=blue!50, fill=blue!10, rounded corners=5pt, ellipse, dotted, inner sep=10pt},
        p_rank0/.style={circle, draw=black, fill=white, inner sep=0pt, minimum size=4pt},
        p_rank1/.style={circle, draw=black, fill=gray!50, inner sep=0pt, minimum size=5pt},
        p_rank2/.style={circle, draw=black, fill=gray!70, inner sep=0pt, minimum size=5pt},
        p_line/.style={thin, gray, -latex}
    }

    % --- SUBFIGURE 1 ---
    % Reduced width to bring the next figure closer
    \begin{minipage}[b]{0.26\linewidth}
        \centering
        \begin{tikzpicture}[scale=0.8]
            % Tighter clipping/bounding box implicit by coordinate choice
            \fill[hg_edge] (0.5,1.5) -- (-2.5,0) -- (0.5,-1.5) -- cycle;
            \fill[hg_edge] (-0.5,1.5) -- (2.5,0) -- (-0.5,-1.5) -- cycle; 
            \node[hg_node] (v1) at (0,-0.8) {};
            \node[hg_node] (v2) at (0,0.8) {};
            \node[hg_node] (v3) at (-1.8,0) {};
            \node[hg_node] (v4) at (1.8,0) {};
            \node at (-1, 1.2) {\scriptsize $e_1$}; 
            \node at (1, 1.2) {\scriptsize $e_2$}; 
            \node at (0, -1) {\scriptsize $v_1$};
            \node at (0, 1) {\scriptsize $v_2$};
            \node at (-2.2, 0) {\scriptsize $v_3$};
            \node at (2.2, 0) {\scriptsize $v_4$};
            \node at (0, -2.0) {$H$};
        \end{tikzpicture}
    \end{minipage}%
    \hfill
    % --- SUBFIGURE 2 ---
    \begin{minipage}[b]{0.24\linewidth}
        \centering
        \begin{tikzpicture}[scale=0.8]
            \node[p_rank1] (e1) at (0.5, 1.5) {};
            \node[p_rank1] (e2) at (1.5, 1.5) {};
            \node[p_rank0] (pv1) at (0, 0) {};
            \node[p_rank0] (pv2) at (0.66, 0) {};
            \node[p_rank0] (pv3) at (1.32, 0) {};
            \node[p_rank0] (pv4) at (2, 0) {};
            \draw[p_line] (pv1) -- (e1);
            \draw[p_line] (pv2) -- (e1);
            \draw[p_line] (pv3) -- (e1);
            \draw[p_line] (pv1) -- (e2);
            \draw[p_line] (pv2) -- (e2);
            \draw[p_line] (pv4) -- (e2);
            \node at (0.5, 1.9) {\scriptsize $e_1$}; 
            \node at (1.5, 1.9) {\scriptsize $e_2$}; 
            \node at (0, -0.4) {\scriptsize $v_1$};
            \node at (0.66, -0.4) {\scriptsize $v_2$};
            \node at (1.32, -0.4) {\scriptsize $v_3$};
            \node at (2, -0.4) {\scriptsize $v_4$};
            \node at (1, -1.2) {$I(H)$};
        \end{tikzpicture}
    \end{minipage}%
    \hfill
    % --- SUBFIGURE 3 ---
    \begin{minipage}[b]{0.48\linewidth}
        \centering
        \begin{tikzpicture}[scale=0.8]
            \node[p_rank2] (e1123) at (-2, 2.5) {};
            \node[p_rank2] (e2124) at (2, 2.5) {};
            \node[p_rank1] (e112) at (-3.5, 1.2) {};
            \node[p_rank1] (e113) at (-2, 1.2) {};
            \node[p_rank1] (e123) at (-0.5, 1.2) {};
            \node[p_rank1] (e224) at (0.5, 1.2) {};
            \node[p_rank1] (e214) at (2, 1.2) {};
            \node[p_rank1] (e212) at (3.5, 1.2) {};

            \node[p_rank0] (v3) at (-3, 0) {};
            \node[p_rank0] (v1) at (-1, 0) {};
            \node[p_rank0] (v2) at (1, 0) {};
            \node[p_rank0] (v4) at (3, 0) {};

            \draw[p_line] (v1) -- (e112); \draw[p_line] (v1) -- (e113); \draw[p_line] (v1) -- (e212); \draw[p_line] (v1) -- (e214);
            \draw[p_line] (v2) -- (e112); \draw[p_line] (v2) -- (e123); \draw[p_line] (v2) -- (e224); \draw[p_line] (v2) -- (e212);
            \draw[p_line] (v3) -- (e113); \draw[p_line] (v3) -- (e123);
            \draw[p_line] (v4) -- (e214); \draw[p_line] (v4) -- (e224);

            \draw[p_line] (e112) -- (e1123); \draw[p_line] (e113) -- (e1123); \draw[p_line] (e123) -- (e1123);
            \draw[p_line] (e212) -- (e2124); \draw[p_line] (e214) -- (e2124); \draw[p_line] (e224) -- (e2124);

            \node at (-3, -0.4) {\scriptsize $v_3$};
            \node at (-1, -0.4) {\scriptsize $v_1$};
            \node at (1, -0.4) {\scriptsize $v_2$};
            \node at (3, -0.4) {\scriptsize $v_4$};

            \node at (-4.2, 1.0) {\tiny $\{v_1, v_2\}_{e_1}$};
            \node at (3.2, 1.5) {\tiny $\{v_1, v_2\}_{e_2}$};
            
            \node at (-2, 2.9) {\tiny $\{v_1,v_2, v_3\}_{e_1}$};
            \node at (2, 2.9) {\tiny $\{v_1, v_2, v_4 \}_{e_2}$};
            \node at (0, -1.2) {$S(H)$};
        \end{tikzpicture}
    \end{minipage}

    \caption{Comparison of lifting functors on a sample hypergraph $H$ (left). The incidence poset $I(H)$ (center) treats hyperedges as dimension 1 objects, while the symmetric simplicial complex $S(H)$ (right) explicitly represents all sub-relations.}
    \label{fig:twofunctors}
\end{figure}
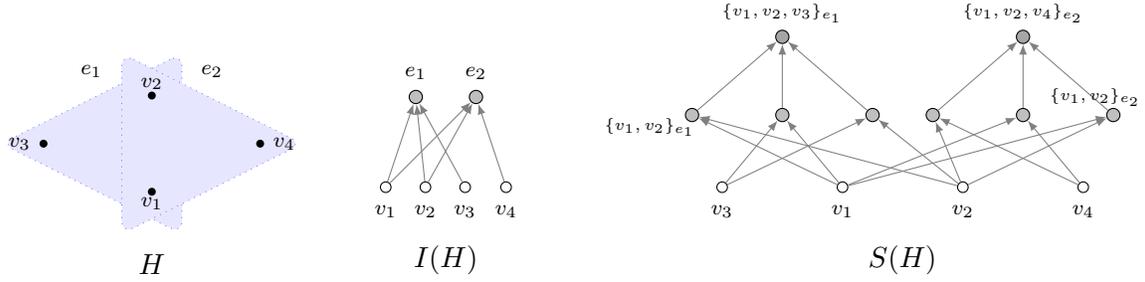

\subsubsection{The Incidence Poset Functor ($I$)}

The simplest way to lift a hypergraph is to treat hyperedges as 1-dimensional objects, nodes as 0-dimensional objects, and the covering relation $v \prec e$ if and only if the node $v$ is contained in the hyperedge $e$. This yields the incidence poset functor.

\begin{definition}[Incidence Poset Functor]\label{def:incidence_functor}The incidence poset functor $I: \mathbf{Hyp} \to \mathbf{Poset}$ is defined as follows: For an object $H=(V,E,f)$, $I(H)$ is the incidence poset $(V \amalg E, \prec)$ where relations are defined by node-edge incidence. For a morphism $\phi=(a,b): H \to H'$, the map $I(\phi)$ sends $v \in V$ to $a(v)$ and $e \in E$ to $b(e)$, preserving the poset structure.
\end{definition}

Under the functor $I$, the hypergraph is mapped to a 2-level graded poset with nodes at dimension 0 and hyperedges at dimension 1. The color refinement in the resulting graded poset corresponds to a natural bipartite interaction: Node updates aggregate colors from the co-boundary $\mathcal{C}(v)$ and upper-adjacency $\mathcal{N}_{\uparrow}(v)$ (as boundary and lower-adjacency sets are empty for dimension 0 nodes). Conversely, hyperedge updates aggregate colors from the boundary $\mathcal{B}(e)$ and lower-adjacency $\mathcal{N}_{\downarrow}(e)$ (as co-boundary and upper-adjacency sets are empty for dimension 1 hyperedges).

Structurally, the architecture of $I$-CatWL mirrors the bipartite interaction scheme underlying the standard Hypergraph Weisfeiler-Lehman (HWL) test \citep{feng2024hypergraph, zhangimproved}. However, deriving it via the incidence functor $I$ enforces a richer information flow. While HWL typically utilizes only a subset of possible interactions, our $I$-CatWL framework inherently aggregates the colors of 4-adjacencies (boundary, co-boundary, lower, and upper). Because the standard HWL update is effectively a restriction of this complete adjacency aggregation, our method captures a superset of the geometric information. Consequently, as we formally establish in the next theorem (proof deferred to Appendix~\ref{app:comparisonofHWLCatWL}), $I$-CatWL is not less powerful than HWL.

\begin{theorem}
\label{thm:ICatWLnotlesspowerfulGWL}
The $I$-CatWL is not less powerful than HWL.
\end{theorem}

\subsubsection{The Symmetric Simplicial Complex Functor ($S$)}

Motivated by recent work on higher-order learning on hypergraphs \citep{choi2025hypergraph}, we consider an alternative lifting strategy that explicitly preserves the internal subset geometry of each hyperedge. The symmetric simplicial complex functor $S$ maps a hypergraph to a structure where every subset of a hyperedge becomes a distinct geometric object, ordered by inclusion. This naturally assigns higher dimensions to larger hyperedges, capturing the full lattice of sub-relations. 

\begin{definition}[Symmetric Simplicial Complex Functor $S$]
\label{def:symmetric_functor}
Let $H=(V,E,f)$ be a hypergraph. The symmetric simplicial complex of $H$, denoted by $S(H)$, consists of 0-simplices $S(H)_0 := \{ (\{v\}, v) \mid v \in V(H) \}$ and $n$-simplices $S(H)_n := \{ (\{v_{i_0}, \dots, v_{i_n} \}, e) \mid \{v_{i_0}, \dots, v_{i_n}\} \subset f(e) \}$ for $n \ge 1$. For notation, we write $\{v_{i_0}, \dots, v_{i_n}\}_e$ for $(\{v_{i_0}, \dots, v_{i_n} \}, e)$ and $\{v\}_v$ for $(\{v\}, v)$. This set admits a graded poset structure defined by two conditions: (1) The partial order satisfies $\{v\}_v < \sigma_e$ if $v \in e$, and $\sigma_e < \tau_{e'}$ if $e=e'$ and the underlying set of $\sigma$ is a proper subset of $\tau$; (2) The dimension function is given by $\dim(\{v\}_v)=0$ and $\dim(\sigma_e)=n$ for $\sigma_e \in S(H)_n$.
\end{definition}

Figure~\ref{fig:twofunctors} illustrates symmetric simplicial complex: while the incidence poset $I(H)$ flattens the structure into a simple bipartite graph, the symmetric simplicial complex $S(H)$ reveals the intricate hierarchy of simplices nested within the same hyperedges, such as the triangle $\{v_1, v_2, v_3\}_{e_1}$ covering the edge $\{v_1, v_2\}_{e_1}$.

The intuition behind $S$ is to view a hyperedge $e = \{v_0, \dots, v_k\}$ as a formal $k$-simplex, generating a full face lattice of sub-relations. Unlike $I(H)$, where all hyperedges reside at dimension 1, $S(H)$ stratifies the structure by cardinality, placing a hyperedge of size $k+1$ at dimension $k$. The resulting color refinement operates across this multi-level hierarchy: Updates at dimension $k$ aggregate colors from the $(k+1)$-dimensional co-boundary (super-relations) and $(k-1)$-dimensional boundary (sub-relations). This allows for fine-grained interaction modeling, such as a node $v$ distinguishing specific pairwise relationships $\{v, w\}$ within a hyperedge, or hyperedges communicating via specific shared faces to capture intersection geometry.
Analogous to the incidence case, the following theorem confirms that this symmetric simplicial approach is not less powerful than HWL (proof detailed in Appendix~\ref{app:comparisonofHWLCatWL}).

\begin{theorem}
\label{thm:SCatWLnotlesspowerfulGWL}
The $S$-CatWL is not less powerful than HWL.
\end{theorem}

\section{Related work}

\paragraph{GNNs and Topological Lifting}
Standard Message Passing Neural Networks (MPNNs) are widely known to be bounded in expressivity by the 1-dimensional WL (1-WL) test \citet{xu2018powerful, morris2019weisfeiler}. While higher-order $k$-WL variants offer greater discriminative power \citet{morris2023weisfeiler}, they incur prohibitive computational costs. Consequently, recent research has shifted toward topological deep learning, where graphs are lifted into higher-order domains to capture geometric invariants such as cycles.
\citet{bodnar2021weisfeiler} pioneered this by lifting graphs to simplicial complexes, enabling message passing over higher-dimensional simplices. This direction was further expanded by \citet{bodnar2021weisfeiler_cw} to regular CW complexes, allowing for the processing of flexible topological structures. Concurrent works have explored similar simplicial architectures for trajectory prediction and signal processing \citet{roddenberry2021principled, eblisimplicial, bunchsimplicial}. More recently, general frameworks for topological message passing have been proposed to unify these domains \citet{hajij2022simplicial, giusti2024topological}.

\paragraph{Hypergraph Learning and Higher-Order WL}
Hypergraph Neural Networks (HNNs) extend graph learning to capture multi-way interactions. Early approaches relied heavily on spectral methods or clique expansions \citep{yadati2019hypergcn}. \citet{feng2019hypergraph} introduced hypergraph neural networks based on hyperedge convolution to efficiently encode high-order correlations, while \citet{bai2021hypergraph} further proposed hypergraph convolution and hypergraph attention operators to adaptively weight high-order interactions. \citet{dong2020hnhn} designed a hypergraph convolution network with separate hypernode and hyperedge neurons and a flexible normalization scheme that controls the influence of high-cardinality hyperedges and high-degree vertices. To address the variations in propagation rules, \citet{huang2021unignn} and \citet{chien2022you} introduced unified frameworks that generalize arbitrary aggregation functions. More recently, \citet{wang2022equivariant} proposed an equivariant hypergraph diffusion architecture that can provably approximate any continuous equivariant hypergraph diffusion operator, providing a principled perspective on how permutations of nodes and hyperedges constrain message passing. However, the theoretical expressivity of these architectures remains a nuanced subject. \citet{boker2019color} established the fundamental link between hypergraph isomorphism and color refinement, while \citet{feng2024hypergraph} provided algorithms for isomorphism computation. \citet{luoexpressiveness} further characterized the separation power of HNNs against standard graph baselines. Most recently, attention has turned to constructing higher-order WL hierarchies specifically for hypergraphs. \citet{zhangimproved} proposed high-dimensional generalized WL algorithms to enhance expressivity beyond the standard 1-WL reduction.

\section{Numerical Experiments}
\label{sec:experiments}
We empirically evaluate CatMPN instantiated with our two lifting functors on six real-world hypergraph classification benchmarks. 

\subsection{Experimental Setup}
\paragraph{Datasets} 
We validate our framework on six real-world hypergraph benchmarks derived from diverse domains, comprising four variants of \texttt{IMDB} collaboration networks, \texttt{steam\_player}, and \texttt{twitter\_friend}~\cite{feng2019hypergraph}. These datasets are chosen to represent a broad spectrum of topological complexity, specifically regarding hyperedge cardinality. The \texttt{IMDB\_wri} datasets feature a granular structure with small collaboration units where the average hyperedge size is below 5, requiring the model to detect subtle correlations. In contrast, \texttt{twitter\_friend} exhibits dense high-order interactions with many hyperedges per graph, characterized by an average of approximately 84.8 hyperedges per graph and up to 603 hyperedges, while individual hyperedges still capture higher-order interactions with cardinality up to 64. Comprehensive statistics, including distribution details and class balance, are provided in Appendix~\ref{app:dataset}.

% \paragraph{Task and Evaluation} 
% We evaluate our model on the hypergraph classification task. Following the standard evaluation protocol for hypergraph learning with limited data size~\cite{xu2018powerful}, we employ 10-fold cross-validation. Given the small scale of these benchmarks, using a separate test set can lead to high variance and unstable performance estimation. Therefore, strictly following the convention in this domain, we do not use a separate test set. Instead, we monitor the average validation accuracy across folds at each epoch and report the best mean accuracy and standard deviation achieved. This metric intuitively measures the model’s ability to distinguish different types of hypergraphs while ensuring robust performance estimation without overfitting to a specific data split.
\paragraph{Task and Evaluation} 
We evaluate on hypergraph classification using 10-fold cross-validation, following the standard protocol for small-scale hypergraph benchmarks~\cite{xu2018powerful}. We report the best mean validation accuracy and standard deviation across folds.

\paragraph{Baselines} 
We compare our models against a comprehensive set of baselines: 
(1) Standard neural networks: MLP; 
(2) Hypergraph neural networks: HCHA~\citep{bai2021hypergraph}, HNHN~\citep{dong2020hnhn}, HyperGCN~\citep{yadati2019hypergcn}, UniGCNII~\citep{huang2021unignn}, and ED-HNN~\citep{wang2022equivariant}; 
(3) Transformer-based models: AllSetTransformer~\citep{chien2022you}; 
(4) WL-based theoretical models: HIC~\citep{feng2024hypergraph} and 2-FHNN~\citep{zhangimproved}. 
All baselines are evaluated under identical settings to ensure fair comparison.

\paragraph{Model Variants}
We instantiate CatMPN with the incidence poset functor $I$ (Definition~\ref{def:incidence_functor}) and the symmetric simplicial complex functor $S$ (Definition~\ref{def:symmetric_functor}), yielding two main model families: I-HIN and S-HIN, respectively. Within each family, we consider different combinations of message types from Definition~\ref{def:f_catmpn}. 
Our primary variants use upper ($\mathcal{N}_{\uparrow}$) and boundary ($\mathcal{B}$) messages, denoted as (u, b) variants.
This choice is theoretically justified by Lemma~\ref{twocardnalities} in Appendix~\ref{app:equivalenceofadjstrs}, which establishes that the full 4-adjacency update rule has the same expressive power as the reduced rule using only boundary and upper adjacencies—thereby enabling a more efficient architecture without sacrificing expressivity.
We additionally consider (u, b, l) variants that include lower ($\mathcal{N}_{\downarrow}$) messages, which may provide practical benefits despite the theoretical equivalence. For $S$-HIN, hyperedges exceeding cardinality 20 are excluded from lifting, a threshold that accommodates the substantial majority of hyperedges across datasets as supported by the topological statistics in Appendix~\ref{app:stat}. Implementation details, including hyperparameter configurations, are provided in Appendix~\ref{app:implementation}.

\begin{table*}[t]
    \centering
    \caption{Evaluation results of hypergraph classification. The best results are in \textbf{bold}.}
    \label{tab:main_results}
    \resizebox{\textwidth}{!}{%
    \begin{tabular}{c|cccccc} 
        \toprule
         & \textbf{IMDB\_dir\_form} & \textbf{IMDB\_dir\_genre} & \textbf{IMDB\_wri\_form} & \textbf{IMDB\_wri\_genre} & \textbf{steam\_player} & \textbf{twitter\_friend} \\ 
        \midrule
        \textbf{MLP} & 64.42 ± 4.50 & 72.10 ± 2.76 & 54.87 ± 3.81 & 35.24 ± 3.95 & 59.43 ± 2.86 & 62.89 ± 2.76 \\
        \textbf{HCHA} & 61.65 ± 3.52 & 67.80 ± 2.83 & 52.18 ± 3.18 & 37.55 ± 3.81 & 65.94 ± 4.07 & 62.43 ± 3.52 \\
        \textbf{HNHN} & 67.04 ± 3.83 & 78.07 ± 1.29 & 54.30 ± 3.11 & 54.17 ± 3.63 & 63.50 ± 3.70 & 66.95 ± 4.01 \\
        \textbf{HyperGCN} & 62.81 ± 3.30 & 72.91 ± 1.38 & 54.33 ± 3.27 & 39.85 ± 3.50 & 65.98 ± 4.28 & 64.19 ± 6.53 \\
        \textbf{UniGCNII} & 60.55 ± 3.50 & 72.39 ± 2.82 & 47.84 ± 2.41 & 43.78 ± 4.11 & 61.09 ± 4.06 & 61.31 ± 1.15 \\
        \textbf{ED\_HNN} & 60.44 ± 3.24 & 72.73 ± 3.18 & 47.56 ± 2.61 & 30.46 ± 1.40 & 61.07 ± 6.87 & 59.76 ± 1.74 \\
        \textbf{AllSetTransformer} & 60.83 ± 3.49 & 76.59 ± 1.46 & 46.51 ± 2.65 & 40.60 ± 3.00 & 60.51 ± 3.26 & 61.30 ± 1.09 \\
        \textbf{HIC} & 67.30 ± 3.75 & 79.49 ± 1.74 & 56.70 ± 4.15 & 51.45 ± 5.38 & 64.89 ± 4.40 & 68.17 ± 6.46 \\
        \textbf{2-FHNN}\textsuperscript{*} & 68.11 ± 2.46 & 78.52 ± 1.07 & 55.36 ± 3.30 & 45.44 ± 1.24 & \textbf{67.53 ± 1.23} & 62.37 ± 1.91 \\  
        \midrule
        \textbf{$I$-HIN (u, b)} & 68.12 ± 2.66 & \textbf{81.41 ± 2.32} & 58.83 ± 4.37 & 59.73 ± 4.55 & 63.93 ± 3.79 & 72.52 ± 4.74 \\
        \textbf{$I$-HIN (u, b, l)} & \textbf{69.02 ± 3.32} & 80.97 ± 2.43 & \textbf{60.17 ± 2.33} & 59.05 ± 4.05 & 63.71 ± 3.25 & \textbf{74.66 ± 3.64} \\
        \textbf{$S$-HIN (u, b)} & 68.55 ± 3.53 & 80.73 ± 2.30 & 58.86 ± 6.37 & 60.83 ± 4.67 & 62.12 ± 2.79 & 65.57 ± 3.66 \\
        \textbf{$S$-HIN (u, b, l)} & 68.81 ± 3.11 & 80.92 ± 2.44 & 57.24 ± 3.52 & \textbf{61.02 ± 4.09} & 62.65 ± 2.66 & 65.56 ± 3.82 \\
        \bottomrule
        \multicolumn{7}{l}{\footnotesize \textsuperscript{*} Results taken directly from the original paper~\cite{zhangimproved}.}
    \end{tabular}%
    }
\end{table*}

\subsection{Results and Analysis}

We now present the empirical results of CatMPN instantiated with our two categorical functors and analyze them in light of the topological properties of the benchmarks. The evaluation addresses three questions: whether CatMPN improves over standard hypergraph neural networks and WL-based theoretical models, how the choice of functor interacts with the underlying hypergraph topology, and how enriching the message-passing topology with additional categorical adjacencies affects performance.

\paragraph{Overall performance against baselines}
Table~\ref{tab:main_results} shows that CatMPN variants achieve the best performance on five of six benchmarks. On IMDB\_dir\_genre, $I$-HIN surpasses the strongest WL-based baseline HIC by 1.92\%, and on twitter\_friend the margin widens to 6.49\%. The gains are particularly pronounced on the challenging IMDB\_wri\_genre task, where $S$-HIN attains 61.02\%, outperforming HNHN, the strongest baseline on this task, by 6.85\%. These improvements empirically validate Theorems~\ref{thm:ICatWLnotlesspowerfulGWL} and~\ref{thm:SCatWLnotlesspowerfulGWL}: the theoretical expressivity advantage of $F$-CatWL over HWL translates into measurable performance gains. The single exception is steam\_player, where 2-FHNN retains the lead. This dataset exhibits highly uniform hyperedge cardinality with a 99th percentile of only 15, a regime where the additional structural encoding of our functors provides limited benefit over the specialized design of 2-FHNN.
%The single exception is \texttt{steam\_player}, where 2-FHNN retains the lead. We attribute this to the highly uniform hyperedge cardinality in this dataset, whose 99th percentile is only 15, diminishing the benefit of cardinality-aware lifting.

\paragraph{Effect of functor and hypergraph topology}
The relative performance of the two functors depends on hypergraph topology, driven primarily by graph-level density rather than hyperedge cardinality. On dense benchmarks containing many hyperedges per graph, $I$-HIN dominates: it outperforms $S$-HIN by 9.09\% on twitter\_friend, which averages 84.8 hyperedges per graph, and by 1.28\% on steam\_player, which averages 46.4. We attribute this gap to the nature of discriminative signals. When graphs contain numerous hyperedges, classification relies on global connectivity patterns across the hypergraph, which the bipartite message passing of $I$-HIN aggregates efficiently. The symmetric simplicial complex $S(H)$ introduces combinatorial complexity without proportional representational gain in this regime.

Conversely, on MDB\_wri\_genre, $S$-HIN achieves the best overall result with 61.02\% compared to 59.73\% for $I$-HIN. This dataset is characterized by few hyperedges per graph, averaging 4.5, yet large hyperedge sizes with a 99th-percentile cardinality of 36. When hypergraphs are sparse, yet individual hyperedges are structurally rich, the discriminative signal shifts to internal sub-relations within each hyperedge. The simplicial decomposition of $S$-HIN explicitly captures pairwise and higher-order face inclusions that the flat incidence structure of $I$-HIN cannot represent. This finding confirms a key categorical insight: the optimal functor is not universal but depends on whether task-relevant information resides in global hyperedge connectivity or local intersection geometry.

\paragraph{Impact of categorical adjacency design}
Lemma~\ref{twocardnalities} establishes that the (u,b) variants are as expressive as the full 4-adjacency rule, so any performance gap between (u,b) and (u,b,l) reflects optimization dynamics rather than representational capacity. Empirically, $I$-HIN benefits from the additional lower adjacency $\mathcal{N}_{\downarrow}$, gaining 2.14\% on twitter\_friend and 1.34\% on IMDB\_wri\_form. We hypothesize that in the flat two-level incidence poset $I(H)$, lower-adjacency messages provide an explicit shortcut for aggregating information from peer hyperedges that share common nodes, easing optimization even though the same information is in principle recoverable from boundary and upper adjacencies alone.

In contrast, $S$-HIN shows negligible sensitivity to this ablation, with performance differences below 0.6\% on most benchmarks. The hierarchical face-lattice structure of $S(H)$ already supplies direct pathways for the same information, rendering the additional messages redundant in practice. This confirms that, while theoretical expressivity is governed by the choice of functor, practical learning efficiency depends on how explicitly the lifted poset exposes task-relevant structure.

\section{Conclusion}
In this work, we unified higher-order representation learning by introducing the $F$-CatWL framework, establishing graded posets as a universal domain for rigorous, functorial lifting. Applying this to hypergraphs, we derived $F$-HIN, where the message passing topology is strictly dictated by the choice of functor. We theoretically proved that both incidence ($I$) and symmetric simplicial ($S$) instantiations strictly subsume the expressivity of the standard HWL test. Empirically, our models consistently outperform baselines, with $I$-HIN excelling on dense social networks and $S$-HIN on fine-grained structures. These results validate category theory not merely as an abstraction, but as a practical engineering blueprint for designing expressive, topology-aware neural architectures.

% Acknowledgments---Will not appear in anonymized version
\acks{We thank a bunch of people and funding agency.}

\bibliography{reference}

@article{weisfeiler1968reduction,
  title={The reduction of a graph to canonical form and the algebra which appears therein},
  author={Weisfeiler, Boris and Leman, Andrei},
  journal={nti, Series},
  volume={2},
  number={9},
  pages={12--16},
  year={1968}
}

@article{bretto2013hypergraph,
  title={Hypergraph theory},
  author={Bretto, Alain},
  journal={An introduction. Mathematical Engineering. Cham: Springer},
  volume={1},
  pages={209--216},
  year={2013},
  publisher={Springer}
}

@article{arvind2020weisfeiler,
  title={On weisfeiler-leman invariance: Subgraph counts and related graph properties},
  author={Arvind, Vikraman and Fuhlbr{\"u}ck, Frank and K{\"o}bler, Johannes and Verbitsky, Oleg},
  journal={Journal of Computer and System Sciences},
  volume={113},
  pages={42--59},
  year={2020},
  publisher={Elsevier}
}

@article{chen2020can,
  title={Can graph neural networks count substructures?},
  author={Chen, Zhengdao and Chen, Lei and Villar, Soledad and Bruna, Joan},
  journal={Advances in neural information processing systems},
  volume={33},
  pages={10383--10395},
  year={2020}
}

@article{wu2023simplicial,
  title={Simplicial complex neural networks},
  author={Wu, Hanrui and Yip, Andy and Long, Jinyi and Zhang, Jia and Ng, Michael K},
  journal={IEEE Transactions on Pattern Analysis and Machine Intelligence},
  volume={46},
  number={1},
  pages={561--575},
  year={2023},
  publisher={IEEE}
}

@book{mac1998categories,
  title={Categories for the working mathematician},
  author={Mac Lane, Saunders},
  volume={5},
  year={1998},
  publisher={Springer Science \& Business Media}
}

@book{stanley2011enumerative, 
    place={Cambridge}, 
    edition={2}, 
    series={Cambridge Studies in Advanced Mathematics}, 
    title={Enumerative Combinatorics}, 
    publisher={Cambridge University Press}, 
    author={Stanley, Richard P.}, 
    year={2011}, 
    collection={Cambridge Studies in Advanced Mathematics}
}

@inproceedings{xu2018powerful,
  title={How powerful are graph neural networks?},
  author={Xu, Keyulu and Hu, Weihua and Leskovec, Jure and Jegelka, Stefanie},
  booktitle={International Conference on Learning Representations},
  year={2019}
}

@inproceedings{morris2019weisfeiler,
  title={Weisfeiler and {L}eman go neural: Higher-order graph neural networks},
  author={Morris, Christopher and Ritzert, Martin and Fey, Matthias and Hamilton, William L and Lenssen, Jan Eric and Rattan, Gaurav and Grohe, Martin},
  booktitle={Proceedings of the AAAI Conference on Artificial Intelligence},
  volume={33},
  pages={4602--4609},
  year={2019}
}

@inproceedings{bodnar2021weisfeiler,
  title={Weisfeiler and {L}ehman go topological: Message passing simplicial networks},
  author={Bodnar, Cristian and Frasca, Fabrizio and Wang, Yuguang and Otter, Nina and Montufar, Guido F and Li{\'o}, Pietro and Bronstein, Michael},
  booktitle={International Conference on Machine Learning},
  pages={1026--1037},
  year={2021},
  organization={PMLR}
}

@inproceedings{bunchsimplicial,
  title={Simplicial 2-Complex Convolutional Neural Networks},
  author={Bunch, Eric and You, Qian and Fung, Glenn and Singh, Vikas},
  booktitle={TDA $\{$$\backslash$\&$\}$ Beyond}
}

@inproceedings{eblisimplicial,
  title={Simplicial Neural Networks},
  author={Ebli, Stefania and Defferrard, Micha{\"e}l and Spreemann, Gard},
  booktitle={TDA $\{$$\backslash$\&$\}$ Beyond}
}

@inproceedings{roddenberry2021principled,
  title={Principled simplicial neural networks for trajectory prediction},
  author={Roddenberry, T Mitchell and Glaze, Nicholas and Segarra, Santiago},
  booktitle={International Conference on Machine Learning},
  pages={9020--9029},
  year={2021},
  organization={PMLR}
}

@inproceedings{hajij2022simplicial,
  title={Simplicial complex representation learning},
  author={Hajij, Mustafa and Zamzmi, Ghada and Papamarkou, Theodore and Maroulas, Vasileios and Cai, Xuanting},
  booktitle={Machine Learning on Graphs (MLoG) Workshop at 15th ACM International WSDM (2022) Conference},
  year={2022}
}

@inproceedings{bodnar2021weisfeiler_cw,
  title={Weisfeiler and {L}ehman go cellular: {CW} networks},
  author={Bodnar, Cristian and Frasca, Fabrizio and Otter, Nina and Wang, Yuguang and Li{\'o}, Pietro and Montufar, Guido F and Bronstein, Michael},
  booktitle={Advances in Neural Information Processing Systems},
  volume={34},
  pages={2625--2640},
  year={2021}
}

@inproceedings{giusti2024topological,
  title={Topological message passing for higher-order and long-range interactions},
  author={Giusti, Lorenzo and Reu, Teodora and Ceccarelli, Francesco and Bodnar, Cristian and Li{\`o}, Pietro},
  booktitle={2024 International Joint Conference on Neural Networks (IJCNN)},
  pages={1--8},
  year={2024},
  organization={IEEE}
}

@article{morris2023weisfeiler,
  title={Weisfeiler and leman go machine learning: The story so far},
  author={Morris, Christopher and Lipman, Yaron and Maron, Haggai and Rieck, Bastian and Kriege, Nils M and Grohe, Martin and Fey, Matthias and Borgwardt, Karsten},
  journal={Journal of Machine Learning Research},
  volume={24},
  number={333},
  pages={1--59},
  year={2023}
}

@inproceedings{feng2019hypergraph,
  title={Hypergraph neural networks},
  author={Feng, Yifan and You, Haoxuan and Zhang, Zizhao and Ji, Rongrong and Gao, Yue},
  booktitle={Proceedings of the AAAI conference on artificial intelligence},
  volume={33},
  number={01},
  pages={3558--3565},
  year={2019}
}

@article{yadati2019hypergcn,
  title={Hypergcn: A new method for training graph convolutional networks on hypergraphs},
  author={Yadati, Naganand and Nimishakavi, Madhav and Yadav, Prateek and Nitin, Vikram and Louis, Anand and Talukdar, Partha},
  journal={Advances in neural information processing systems},
  volume={32},
  year={2019}
}

@inproceedings{huang2021unignn,
  title={UniGNN: a Unified Framework for Graph and Hypergraph Neural Networks},
  author={Huang, Jing and Yang, Jie},
  booktitle={Proceedings of the Thirtieth International Joint Conference on Artificial Intelligence},
  pages={2563--2569},
  year={2021},
  organization={International Joint Conferences on Artificial Intelligence Organization}
}

@inproceedings{chien2022you,
  title={You are {AllSet}: A multiset function framework for hypergraph neural networks},
  author={Chien, Eli and Pan, Chao and Peng, Jianhao and Milenkovic, Olgica},
  booktitle={International Conference on Learning Representations},
  year={2022}
}

@article{feng2024hypergraph,
  title={Hypergraph isomorphism computation},
  author={Feng, Yifan and Han, Jiashu and Ying, Shihui and Gao, Yue},
  journal={IEEE Transactions on Pattern Analysis and Machine Intelligence},
  volume={46},
  number={5},
  pages={3880--3896},
  year={2024},
  publisher={IEEE}
}

@inproceedings{boker2019color,
  title={Color refinement, homomorphisms, and hypergraphs},
  author={B{\"o}ker, Jan},
  booktitle={International Workshop on Graph-Theoretic Concepts in Computer Science},
  pages={338--350},
  year={2019},
  organization={Springer}
}

@inproceedings{zhangimproved,
  title={Improved Expressivity of Hypergraph Neural Networks through High-Dimensional Generalized Weisfeiler-Leman Algorithms},
  author={Zhang, Detian and Zhang, Chengqiang and Rao, Yanghui and Qing, Li and Zhu, Chunjiang},
  booktitle={Forty-second International Conference on Machine Learning},
  year={2025}
}

@inproceedings{luoexpressiveness,
  title={On the Expressiveness and Generalization of Hypergraph Neural Networks},
  author={Luo, Zhezheng and Mao, Jiayuan and Tenenbaum, Joshua B and Kaelbling, Leslie Pack},
  booktitle={The First Learning on Graphs Conference},
  year={2022}
}

@article{choi2025hypergraph,
  title={Hypergraph Neural Sheaf Diffusion: A Symmetric Simplicial Set Framework for Higher-Order Learning},
  author={Choi, Seongjin and Kim, Gahee and Oh, Yong-Geun},
  journal={IEEE Access},
  year={2025},
  publisher={IEEE}
}

@article{bai2021hypergraph,
  title={Hypergraph convolution and hypergraph attention},
  author={Bai, Song and Zhang, Feihu and Torr, Philip HS},
  journal={Pattern Recognition},
  volume={110},
  pages={107637},
  year={2021},
  publisher={Elsevier}
}

@article{dong2020hnhn,
  title={Hnhn: Hypergraph networks with hyperedge neurons},
  author={Dong, Yihe and Sawin, Will and Bengio, Yoshua},
  journal={arXiv preprint arXiv:2006.12278},
  year={2020}
}

@article{wang2022equivariant,
  title={Equivariant hypergraph diffusion neural operators},
  author={Wang, Peihao and Yang, Shenghao and Liu, Yunyu and Wang, Zhangyang and Li, Pan},
  journal={arXiv preprint arXiv:2207.06680},
  year={2022}
}

\newpage
\appendix

\section{Formal Definitions of Lifting Functors}
\label{app:Formaldenofliftingfunctors}

In this section, we formalize the categorical structures and verify the functoriality of our lifting constructions. We denote the category of hypergraphs by $\mathbf{Hyp}$ and the category of graded posets by $\mathbf{Poset}$.

\begin{definition}[Categories $\mathbf{Hyp}$ and $\mathbf{Poset}$]
    An object in $\mathbf{Hyp}$ is a triple $H=(V,E,f)$ where $f: E \to 2^V \setminus \emptyset$. A morphism $\phi=(a,b): H \to H'$ consists of maps $a: V \to V', b: E \to E'$ satisfying $a_* \circ f = f' \circ b$, where $a_*$ is the induced power set map.
    An object in $\mathbf{Poset}$ is a graded poset $(P, \leq, \dim)$. A morphism $g: P \to P'$ is an order-preserving map such that $\dim'(g(p)) \leq \dim(p)$.
\end{definition}

\begin{theorem}
    The Symmetric Simplicial Complex construction $S: \mathbf{Hyp} \to \mathbf{Poset}$ is a functor.
\end{theorem}

\begin{proof}
    Given a morphism of hypergraphs $\phi=(a,b): H \to H'$, we define the map $S(\phi): S(H) \to S(H')$ on simplices by $S(\phi)(\{v_{i_0}, \dots, v_{i_n}\}_e) \coloneqq \{a(v_{i_0}), \dots, a(v_{i_n}) \}_{b(e)}$. This map is a valid morphism in $\mathbf{Poset}$ because:
    \begin{enumerate}[leftmargin=*,topsep=0pt,itemsep=0ex]
        \item If $\sigma_e \subset \tau_e$ (i.e., $\{v_{i_k}\} \subset \{v_{j_k}\}$), then by definition of the set map $a$, $\{a(v_{i_k})\} \subset \{a(v_{j_k})\}$, implying $S(\phi)(\sigma_e) \subset S(\phi)(\tau_e)$ in $b(e)$.
        \item Since $|a(\sigma)| \leq |\sigma|$, we have $\dim(S(\phi)(\sigma_e)) = |a(\sigma)|-1 \leq |\sigma|-1 = \dim(\sigma_e)$.
    \end{enumerate}
    Functoriality follows from the component-wise composition of $a$ and $b$.
\end{proof}

\begin{remark}[Standard Geometric Posets]
    Our framework naturally encompasses standard geometric structures. The incidence poset $I(H)$ is a valid graded poset where nodes are assigned dimension 0 and hyperedges dimension 1. Similarly, any simplicial complex $X$ forms a graded poset under inclusion with $\dim(\sigma)=|\sigma|-1$, and the face poset of a regular CW complex is graded by cell dimension.
\end{remark}

\section{Equivalence of Adjacency Structures}
\label{app:equivalenceofadjstrs}
In this section, we establish that the full 4-adjacency update rule has the same expressive power as a simplified update rule using only a subset of these adjacencies (typically Boundary and Up, or similar). This result generalizes similar findings in topological and cellular WL tests \citep{bodnar2021weisfeiler, bodnar2021weisfeiler_cw} to our categorical framework, providing theoretical justification for the efficient architecture used in our experiments.

\begin{definition}[Isomorphism Test and Expressivity]
    Let $\mathcal{C}$ be a category. An isomorphism test $A=\{A_t\}_{t \in \mathbb{N}}$ is a sequence of equivalence relations on $\mathcal{C}$ such that (1) $X \cong X' \implies X \cong_{A_t} X'$ for all $t$, and (2) $X \cong_{A_{t+1}} X' \implies X \cong_{A_t} X'$. Two objects are distinguishable by $A$ if $X \not\cong_{A_t} X'$ for some $t$.
    
    Given two tests $A, A'$, we say $A$ is at least as powerful as $A'$ ($A \sqsubseteq A'$) if $X \cong_A X' \implies X \cong_{A'} X'$. We say $A$ is not less powerful than $A'$ if there exist objects $X, X'$ distinguishable by $A$ but not by $A'$ (i.e., $X \not\cong_A X'$ yet $X \cong_{A'} X'$). If $A \sqsubseteq A'$ and $A' \sqsubseteq A$, we say they have the same expressive power. If $A \sqsubseteq A'$ but there exist objects distinguishable by $A$ and not $A'$, then $A$ is strictly more powerful.
\end{definition}

\begin{definition}[Coloring Update Rule]
    \label{defncoloringupdaterule}
    A coloring $c$ on $\mathcal{C}$ assigns a map $c^X: X \to S_c$ to each object $X$. We write $X \cong_c X'$ if the multisets (histograms) of colors of $c^X, c^{X'}$ are identical. A coloring update rule is a sequence $\mathbb{c} := \{\mathbb{c}_t \}_t$ where (1) $\mathbb{c}_0$ is constant, (2) $X \cong_{\mathbb{c}_{t+1}} X' \implies X \cong_{\mathbb{c}_t} X'$, and (3) $X \cong X' \implies X \cong_{\mathbb{c}_t} X'$ for all $t$.
\end{definition}
 
We can easily prove the following lemma by the definition of the isomorphism test.
\begin{lemma}\label{coloringrefinementinducesisomorphismtest}Coloring update rule $\mathbb{c}$ on category $\mathcal{C}$ induces isomorphism test where $X \cong_{\mathbb{c}} X'$ if and only if two coloring histograms of $\mathbb{c}^X_{t}, \mathbb{c}^{X'}_{t}$ are equal.
\end{lemma}

Hence, we abuse the notation of coloring update rule and its associated isomorphism test by Lemma~\ref{coloringrefinementinducesisomorphismtest}.

\begin{definition}[Coloring Refinement]
    A coloring $c$ refines $d$ (denoted $c \sqsubseteq d$) if $c^X(\sigma) = c^{X'}(\sigma') \implies d^X(\sigma) = d^{X'}(\sigma')$ for all inputs. Two colorings are equivalent ($c \equiv d$) if $c \sqsubseteq d$ and $d \sqsubseteq c$.
\end{definition}

\begin{lemma}[Reparameterization Invariance]
    \label{reparametrizationofisomorphismtest}
    Let $\mathbb{c}=\{\mathbb{c}_t\}_t$ be a coloring update rule and $\varphi: \mathbb{N} \to \mathbb{N}$ be a strictly increasing function with $\varphi(t) \geq t+1$. The reparameterized sequence $\varphi(\mathbb{c}) := \{\mathbb{c}_{\varphi(t)}\}_t$ forms a valid isomorphism test with the same expressive power: $X \cong_{\mathbb{c}} X' \iff X \cong_{\varphi(\mathbb{c})} X'$.
\end{lemma}

\begin{proof}
    Validity follows from the properties of $\mathbb{c}$. Condition (1) $X \cong X' \implies X \cong_{\mathbb{c}_{\varphi(t)}} X'$ holds by definition. Condition (2) follows from monotonicity: $X \cong_{\mathbb{c}_{\varphi(t+1)}} X' \implies X \cong_{\mathbb{c}_{\varphi(t)}} X'$ because $\varphi(t+1) > \varphi(t)$ and colorings refine over time ($\mathbb{c}_{k+1} \sqsubseteq \mathbb{c}_k$).
    The equivalence $X \cong_{\mathbb{c}} X' \iff X \cong_{\varphi(\mathbb{c})} X'$ is immediate: The forward direction ($\implies$) is trivial as $\{\mathbb{c}_{\varphi(t)}\}$ is a subsequence of $\{\mathbb{c}_t\}$. The reverse ($\impliedby$) holds because for any $k$, there exists $t$ such that $\varphi(t) \ge k$, implying $X \cong_{\mathbb{c}_{\varphi(t)}} X' \implies X \cong_{\mathbb{c}_k} X'$.
\end{proof}

\begin{lemma}\label{colorrefinementlemma}
Let $X, X'$ be objects in $\mathcal{C}$ with $A \subset X, A' \subset X'$. For two colorings $c, d$ on $\mathcal{C}$ such that $c \sqsubseteq d$, if $\{\!\!\{d^{X}(\sigma) \mid \sigma \in A \}\!\!\} \neq \{\!\!\{ d^{X'}(\sigma') \mid \sigma' \in A' \}\!\!\}$, then $\{\!\!\{c^{X}(\sigma) \mid \sigma \in A \}\!\!\} \neq \{\!\!\{ c^{X'}(\sigma') \mid \sigma' \in A' \}\!\!\}$.
\end{lemma}

\begin{proof}
Suppose $\{\!\!\{d^{X}(\sigma) \mid \sigma \in A \}\!\!\} \neq \{\!\!\{ d^{X'}(\sigma') \mid \sigma' \in A' \}\!\!\}$. Without loss of generality, there exists a color $\mathbb{C}$ that appears more frequently in the first multiset. Let $A^* := (d^X)^{-1}({\mathbb{C}}) \cap A$ and $(A')^* := (d^{X'})^{-1}({\mathbb{C}}) \cap A'$. Since $c \sqsubseteq d$, the set of $c$-values mapping to $\mathbb{C}$ is disjoint from those mapping to any other $d$-color. Thus:
\begin{equation}\label{disjointness}
    \begin{aligned}
    &\{\!\!\{c^X(\sigma) \mid \sigma \in A^* \}\!\!\} \cap \{\!\!\{c^X(\sigma) \mid \sigma \in A \setminus A^* \}\!\!\} = \emptyset, \\
    &\{\!\!\{c^{X'}(\sigma') \mid \sigma' \in (A')^* \}\!\!\} \cap \{\!\!\{c^{X'}(\sigma') \mid \sigma' \in A' \setminus (A')^* \}\!\!\} = \emptyset.
    \end{aligned}
\end{equation}
Now, suppose for contradiction that $\{\!\!\{ c^X(\sigma) \mid \sigma \in A \}\!\!\} = \{\!\!\{ c^{X'}(\sigma') \mid \sigma' \in A' \}\!\!\}$. The disjointness property in Eq.~\eqref{disjointness} implies that the sub-multisets corresponding to $\mathbb{C}$ must match:
\begin{equation}\label{samemultisetforc}
    \{\!\!\{ c^X(\sigma) \mid \sigma \in A^* \}\!\!\} = \{\!\!\{ c^{X'}(\sigma') \mid \sigma' \in (A')^* \}\!\!\}.
\end{equation}
Since $c \sqsubseteq d$, all elements in $A^*$ and $(A')^*$ map to the same $d$-color $\mathbb{C}$. Therefore, the equality of $c$-multisets implies equality of $d$-multisets (which are just constant multisets of $\mathbb{C}$ with sizes $|A^*|$ and $|(A')^*|$):
\begin{equation}\label{samemultisetford}
    \{\!\!\{ d^X(\sigma) \mid \sigma \in A^* \}\!\!\} = \{\!\!\{ d^{X'}(\sigma') \mid \sigma' \in (A')^* \}\!\!\}.
\end{equation}
This implies $|A^*| = |(A')^*|$, contradicting the assumption that $\mathbb{C}$ appears more times in the first multiset. Hence, the $c$-multisets must differ.
\end{proof}

\begin{corollary}\label{coloringrefinementimpliesatleastpowerfulisotest}
For colorings $c, d$ on $\mathcal{C}$ with $c \sqsubseteq d$, if $X \ncong_d X'$ then $X \ncong_c X'$.
\end{corollary}

\begin{proof}
By Definition~\ref{defncoloringupdaterule}, $X \ncong_d X'$ implies the multisets $\{\!\!\{d^{X}(\sigma) \mid \sigma \in X \}\!\!\}$ and $\{\!\!\{ d^{X'}(\sigma') \mid \sigma' \in X' \}\!\!\}$ are distinct. Applying Lemma~\ref{colorrefinementlemma} with $A=X$ and $A'=X'$, it follows that the corresponding $c$-multisets are also distinct, which means $X \ncong_c X'$.
\end{proof}

The injectivity of the hash function implies a fundamental invariant: distinct boundary sizes yield distinct colors.

\begin{lemma}
\label{samecatwlcoloringmusthavesamecardinality}
Let $\mathbb{c}^{F,X}_t$ denote the $F$-CatWL coloring at iteration $t$. For any $\sigma \in F(X)$ and $\sigma' \in F(X')$, if $|\mathcal{B}(\sigma)| \neq |\mathcal{B}(\sigma')|$, then $\mathbb{c}^{F,X}_t(\sigma) \neq \mathbb{c}^{F,X'}_t(\sigma')$ for all $t > 0$.
\end{lemma}

We now establish that under mild conditions on $F$, the full 4-adjacency update rule is equivalent to a reduced update rule using only boundary, upper, and lower adjacencies (omitting co-boundary).

\begin{lemma}\label{threecardnalities}
Let $F: \mathcal{C} \to \mathcal{P}$ be a functor such that $|\mathcal{B}(\sigma)| > 1$ for any $\sigma \in F(X)$ with $\dim(\sigma) > 0$. The $F$-CatWL using the full 4-adjacency update rule has the same expressive power as the restricted test using only boundary, lower, and upper adjacencies (omitting co-boundary).
\end{lemma}

\begin{proof}
Let $\mathbb{a}^F_t$ and $\mathbb{b}^F_t$ denote the coloring update rules using 4-adjacencies and 3-adjacencies, respectively. Clearly $\mathbb{a}^F_t \sqsubseteq \mathbb{b}^F_t$ since $\mathbb{a}$ utilizes strictly more information (co-boundary). It suffices to show $\mathbb{b}^F_{t+1} \sqsubseteq \mathbb{a}^F_t$ by induction on $t$ (applying Lemma~\ref{reparametrizationofisomorphismtest}).

\textbf{Base case ($t=0$):} Trivial since $\mathbb{a}_0$ is constant.

\textbf{Inductive step:} Suppose $\mathbb{b}^F_{k+1} \sqsubseteq \mathbb{a}^F_k$ for $k=0, \dots, t$. Assume $\mathbb{b}^{F,X}_{t+2}(\sigma) = \mathbb{b}^{F,X'}_{t+2}(\sigma')$. By injectivity of the update function, the colors of $\sigma$, $\mathcal{B}(\sigma)$, $\mathcal{N}_{\downarrow}(\sigma)$, and $\mathcal{N}_{\uparrow}(\sigma)$ at time $t+1$ must match those of $\sigma'$. We show this implies the co-boundary colors also match: $\mathbb{b}^{F,X}_{t+1}(\mathcal{C}(\sigma)) = \mathbb{b}^{F,X'}_{t+1}(\mathcal{C}(\sigma'))$.

The equality of upper-adjacency multisets $\mathbb{b}^{F,X}_{t+1}(\mathcal{N}_{\uparrow}(\sigma)) = \mathbb{b}^{F,X'}_{t+1}(\mathcal{N}_{\uparrow}(\sigma'))$ implies equality of the set of tuples $\{ (b^{F,X}_{t+1}(\delta), b^{F,X}_{t+1}(\sigma)) \mid \delta \in \mathcal{C}(\sigma) \}$.
Using Lemma~\ref{samecatwlcoloringmusthavesamecardinality}, we can partition these multisets by the size of the boundary $|\mathcal{B}(\delta)|=n$. For a fixed $n$, the color $b^{F,X}_{t+1}(\delta)$ appears in the upper-adjacency multiset of $\sigma$ exactly $n-1$ times (since $\delta$ has $n$ boundary elements, one of which is $\sigma$).
Since $|\mathcal{B}(\delta)| > 1$, we have $n \ge 2$, so $n-1 \ge 1$. We can thus uniquely recover the multiset of co-boundary neighbors by dividing multiplicities by $n-1$. This implies:
\begin{equation*}
    \{\!\!\{ b^{F,X}_{t+1}(\delta) \mid \delta \in \mathcal{C}(\sigma) \}\!\!\} = \{\!\!\{ b^{F,X'}_{t+1}(\delta') \mid \delta' \in \mathcal{C}(\sigma') \}\!\!\}.
\end{equation*}
Thus, all four adjacency inputs match. By the induction hypothesis, $\mathbb{a}^F_t$ is refined by $\mathbb{b}^F_{t+1}$, so the inputs to the $\mathbb{a}$-update at $t$ are identical for $\sigma$ and $\sigma'$. Consequently, $\mathbb{a}^{F,X}_{t+1}(\sigma) = \mathbb{a}^{F,X'}_{t+1}(\sigma')$. This establishes that $\mathbb{b}^{F}_{t+2} \sqsubseteq \mathbb{a}^{F}_{t+1}$.
\end{proof}

Using Lemma~\ref{threecardnalities}, we can show that $F$-CatWL with boundary, upper adjacencies has the same expressive power as $F$-CatWL with 4-adjacencies.

\begin{lemma}\label{twocardnalities}
Let $F: \mathcal{C} \to \mathcal{P}$ be a functor such that $|\mathcal{B}(\sigma)| > 1$ for any $\sigma \in F(X)$ with $\dim(\sigma) > 0$. The $F$-CatWL test using all four adjacencies (boundary, co-boundary, lower, upper) has the same expressive power as the restricted test using only boundary and upper adjacencies (omitting co-boundary and lower).
\end{lemma}

\begin{proof}
Let $\mathbb{b}^{F}_t$ denote the restricted coloring update rule (using boundary and upper adjacencies) and $\mathbb{a}^{F}_t$ denote the rule using all four adjacencies. Clearly $\mathbb{a}^{F}_t \sqsubseteq \mathbb{b}^{F}_t$ since $\mathbb{a}$ utilizes strictly more information (co-boundary and lower adjacencies). It suffices to show $\mathbb{b}^{F}_{2t+1} \sqsubseteq \mathbb{a}^{F}_t$ by induction on $t$.

The base case ($t=0$) is trivial since $\mathbb{a}^{F}_0$ is constant. For the inductive step, assume $\mathbb{b}^{F}_{2k+1} \sqsubseteq \mathbb{a}^{F}_k$ for $k=0, \dots, t$. Suppose for contradiction that $\mathbb{b}^{F,X}_{2t+3}(\sigma) = \mathbb{b}^{F,X'}_{2t+3}(\sigma')$ but the lower adjacency multisets differ: $\mathbb{b}^{F,X}_{2t+1}(\mathcal{N}_{\downarrow}(\sigma)) \neq \mathbb{b}^{F,X'}_{2t+1}(\mathcal{N}_{\downarrow}(\sigma'))$.

By injectivity of the update function, the colors of $\sigma$, $\mathcal{B}(\sigma)$, and $\mathcal{N}_{\uparrow}(\sigma)$ at time $2t+1$ must match. Since the lower adjacency multisets differ, there exists some color pair $(\mathbb{C}_0, \mathbb{C}_1)$ appearing with different multiplicities in the lower adjacencies of $\sigma$ and $\sigma'$. We define the auxiliary multiset $A_X(\delta)$ for any $\delta$ as:
\begin{equation*}
    A_X(\delta) := \{\!\!\{ (\mathbb{b}_{2t+1}^{I,X}(\psi), \mathbb{b}_{2t+1}^{I,X}(\delta)) \mid \psi \in \mathcal{C}(\delta) \}\!\!\}
\end{equation*}
and aggregate these over the boundary of $\gamma$:
\begin{equation*}
    C_X(\gamma) := \{\!\!\{ |A_X(\delta)| \mid \delta \in \mathcal{B}(\gamma) \}\!\!\}.
\end{equation*}
The sum of elements in $C_X(\sigma)$ corresponds to the total count of tuples in $\mathbb{b}_{2t+1}(\mathcal{N}_{\downarrow}(\sigma))$. A discrepancy in the lower adjacency multisets thus implies $C_X(\sigma) \neq C_{X'}(\sigma')$, which contradicts the assumed equality of $\sigma$ and $\sigma'$ colors at step $2t+3$, as this information would be propagated through the boundary aggregation.

\begin{proposition}\label{technicalproposition}
If $C_X(\sigma) \neq C_{X'}(\sigma')$, then $\mathbb{b}^{F,X}_{2t+3}(\sigma) \neq \mathbb{b}^{F,X'}_{2t+3}(\sigma')$.
\end{proposition}

\begin{proof}
Define a new coloring $c_X(\delta) := |A_X(\delta)|$. We first show that $\mathbb{b}_{2t+2}^{F} \sqsubseteq c$. Let $\delta \in F(X)$ and $\delta' \in F(X')$ be such that $|A_X(\delta)| > |A_{X'}(\delta')|$. If $\mathbb{b}^{F,X'}_{2t+1}(\delta') \neq \mathbb{C}_1$, then since $|A_X(\delta)| > 0$, we have $\mathbb{b}^{F,X}_{2t+1}(\delta) = \mathbb{C}_1$, so their colors differ at step $2t+1$ (and thus at $2t+2$). If $\mathbb{b}^{F,X'}_{2t+1}(\delta') = \mathbb{C}_1$, then $\mathbb{b}^{F,X}_{2t+1}(\delta) = \mathbb{C}_1$ as well. By Lemma~\ref{samecatwlcoloringmusthavesamecardinality} and the assumption on $F$, $|\mathcal{B}(\delta)| = |\mathcal{B}(\delta')| = n > 1$. The color $\mathbb{C}_0$ appears $|A_X(\delta)| \times (n-1)$ times in the upper-adjacency multiset of $\delta$ and $|A_{X'}(\delta')|\times (n-1)$ in the upper-adjacency multiset of $\delta'$. Since $n > 1$ and the counts differ, the upper-adjacency multisets differ, implying $\mathbb{b}_{2t+2}^{F,X}(\delta) \neq \mathbb{b}_{2t+2}^{F,X'}(\delta')$.

Having established $\mathbb{b}_{2t+2}^{F} \sqsubseteq c$, we apply the refinement property to the inequality $C_X(\sigma) \neq C_{X'}(\sigma')$. The difference in $c$-multisets over the boundary implies a difference in $\mathbb{b}_{2t+2}^{F}$-multisets over the boundary. Consequently, the inputs for the update rule at $\sigma$ and $\sigma'$ differ, yielding $\mathbb{b}_{2t+3}^{F,X}(\sigma) \neq \mathbb{b}_{2t+3}^{F,X'}(\sigma')$.
\end{proof}

Proposition~\ref{technicalproposition} establishes that $C_X(\sigma) \neq C_{X'}(\sigma')$ implies $\mathbb{b}_{2t+3}^{F,X}(\sigma) \neq \mathbb{b}_{2t+3}^{F,X'}(\sigma')$, which contradicts our initial assumption. Therefore, we must have equality of the lower-adjacency multisets: $\mathbb{b}^{F,X}_{2t+1}(\mathcal{N}_{\downarrow}(\sigma)) = \mathbb{b}^{F,X'}_{2t+1}(\mathcal{N}_{\downarrow}(\sigma'))$. Since we have now shown that $\sigma$ and $\sigma'$ have identical $\mathbb{b}_{2t+1}^F$-color multisets for all four adjacencies, and since $\mathbb{b}_{2t+1}^F \sqsubseteq \mathbb{a}_t^F$ (by the induction hypothesis), it follows that their $\mathbb{a}_t^F$-color multisets are also identical. Consequently, $\mathbb{a}^{F,X}_{t+1}(\sigma) = \mathbb{a}^{F,X'}_{t+1}(\sigma')$, which completes the inductive step that $\mathbb{b}^{F}_{2t+3} \sqsubseteq \mathbb{a}^{F}_{t+1}$.
\end{proof}

\section{$F$-HIN Architecture}
\label{app:FHINarchitecture}
Given a functor $F: \mathbf{Hyp} \to \mathbf{Poset}$, an $F$-CatMPN is termed an $F$-HIN if its message passing functions take the following specific forms:

\begin{equation}
    \begin{aligned}
    \mathbb{m}_{t}^{F,X}(\mathcal{B}(\sigma)) &= \mathrm{MLP}^{F}_{t, \mathcal{B}}\left((1+\epsilon_{\mathcal{B}})\mathbb{x}^{F,X}_t(\sigma) + \sum_{\tau \in \mathcal{B}(\sigma)}\mathbb{x}^{F,X}_t(\tau)  \right) \\
    \mathbb{m}_{t}^{F,X}(\mathcal{C}(\sigma)) &= \mathrm{MLP}^{F}_{t, \mathcal{C}}\left((1+\epsilon_{\mathcal{C}})\mathbb{x}^{F,X}_t(\sigma)+ \sum_{\tau \in \mathcal{C}(\sigma)}\mathbb{x}^{F,X}_t(\tau)  \right) \\
    \mathbb{m}_{t}^{F,X}(\mathcal{N}_{\downarrow}(\sigma)) &= \mathrm{MLP}^{F}_{t, \downarrow}\left((1+\epsilon_{\downarrow})\mathbb{x}^{F,X}_t(\sigma) + \sum_{(\tau, \sigma') \in \mathcal{N}_{\downarrow}(\sigma)}\mathrm{MLP}_{t,D}(\mathbb{x}^{F,X}_t(\sigma') \mathbin{\|} \mathbb{x}^{F,X}_t(\tau))  \right) \\
    \mathbb{m}_{t}^{F,X}(\mathcal{N}_{\uparrow}(\sigma)) &= \mathrm{MLP}^{F}_{t, \uparrow}\left((1+\epsilon_{\uparrow})\mathbb{x}^{F,X}_t(\sigma) + \sum_{(\tau, \sigma') \in \mathcal{N}_{\uparrow}(\sigma)}\mathrm{MLP}_{t,U}(\mathbb{x}^{F,X}_t(\sigma') \mathbin{\|} \mathbb{x}^{F,X}_t(\tau))  \right)
    \end{aligned}
\end{equation}
Here, $\mathrm{MLP}^{F}_{t, \bullet}$ denotes a 2-layer Perceptron and $\mathrm{MLP}_{t, \bullet}$ denotes a 1-layer Perceptron. The learnable parameters include the weights of these MLPs and the scalars $\epsilon_{\mathcal{B}}, \epsilon_{\mathcal{C}}, \epsilon_{\downarrow}, \epsilon_{\uparrow}$.

\section{Comparison of HWL, $I$-CatWL, and $S$-CatWL}
\label{app:comparisonofHWLCatWL}
In this section, we prove that $I$-CatWL, $S$-CatWL are not less powerful than HWL.

\begin{theorem}\label{IWLnotlesspowerfulthanGWL}$I$-CatWL is not less powerful than HWL.\end{theorem}

\begin{proof}Consider the hypergraphs $H$ and $H'$ in Figure~\ref{fig:notlesspowerfulexamplehypergraphs} with initial constant coloring $\ast$. As shown by \citet{zhangimproved}, $H$ and $H'$ are isomorphic with respect to the HWL test. To demonstrate that they are non-isomorphic under $I$-CatWL, it suffices to show that the colorings $\mathbb{c}^{I,H}_{1}(e_2)$ and $\mathbb{c}^{I,H'}_{1}(e'_2)$ are distinct, where $e_2$ and $e'_2$ are the unique hyperedges of cardinality 2 in $H$ and $H'$ respectively.

We first examine the colors aggregated from the 4-adjacencies of $e_2$ in $H$ at iteration 1. The boundary colors are $\mathbb{c}^{I,H}_{0}(\mathcal{B}(e_2)) = \{\!\!\{ \mathbb{c}^{I,H}_0(v_2), \mathbb{c}^{I,H}_0(v_3) \}\!\!\} = \{\!\!\{ \ast, \ast \}\!\!\}$, and the lower-adjacency colors are $\mathbb{c}^{I,H}_{0}(\mathcal{N}_{\downarrow}(e_2)) = \{\!\!\{ \mathbb{c}^{I,H}_0(e_1), \mathbb{c}^{I,H}_0(e_3) \}\!\!\} = \{\!\!\{ \ast, \ast \}\!\!\}$. Both the co-boundary $\mathcal{C}(e_2)$ and upper-adjacency $\mathcal{N}_{\uparrow}(e_2)$ sets are empty, yielding empty color multisets.

In contrast, for the hyperedge $e'_2$ in $H'$, while the boundary colors remain $\mathbb{c}^{I,H'}_{0}(\mathcal{B}(e'_2)) = \{\!\!\{ \ast, \ast \}\!\!\}$, the lower-adjacency differs. Specifically, $\mathbb{c}^{I,H'}_{0}(\mathcal{N}_{\downarrow}(e'_2)) = \{\!\!\{ \mathbb{c}^{I,H'}_0(e'_1) \}\!\!\} = \{\!\!\{ \ast \}\!\!\}$, containing only a single neighbor. The co-boundary and upper-adjacency sets remain empty.

Consequently, applying the update rule yields distinct hash values:
\begin{equation}
\mathbb{c}^{I,H}_{1}(e_2) = \mathrm{HASH}\left(\ast, \{\!\!\{\ast,\ast \}\!\!\}, \emptyset, \{\!\!\{\ast,\ast \}\!\!\}, \emptyset \right) 
\neq 
\mathrm{HASH}\left(\ast, \{\!\!\{\ast,\ast \}\!\!\}, \emptyset, \{\!\!\{\ast \}\!\!\}, \emptyset \right) = \mathbb{c}^{I,H'}_{1}(e'_2).
\end{equation}
By the injectivity of $\mathrm{HASH}$, the resulting color multisets are distinct, meaning $I$-CatWL successfully distinguishes $H$ and $H'$. This confirms that $I$-CatWL is not less powerful than HWL.
\end{proof}

\begin{figure}[t]
    \centering
    
    \begin{tikzpicture}
        % Define styles for hypergraphs
        \tikzstyle{hg_node}=[circle, fill=black, inner sep=0pt, minimum size=3pt]
        \tikzstyle{hg_edge}=[draw=blue!50, fill=blue!10, rounded corners=5pt, ellipse, dotted, inner sep=10pt]

        % Hypergraph X (Top Left)
        \node (X) at (0, 3) {
            \begin{tikzpicture}[scale=0.6]
                % Hyperedge e1 = {v1, v2, v3}
                \fill[hg_edge] (0,4) -- (-3.5,0.2) -- (3.5,0.2) -- cycle;
                \fill[hg_edge] (0,-4) -- (-3.5,-0.2) -- (3.5,-0.2) -- cycle;
                \fill[hg_edge] (1.5,2) -- (2.5,2) -- (2.5,-2) -- (1.5,-2) -- cycle;
                \fill[hg_edge] (-1.5,2) -- (-2.5,2) -- (-2.5,-2) -- (-1.5,-2) -- cycle;
                
                \node[hg_node] (v1) at (0,2.73) {};
                \node[hg_node] (v2) at (2,1) {};
                \node[hg_node] (v3) at (2,-1) {};
                \node[hg_node] (v4) at (0,-2.73) {};
                \node[hg_node] (v5) at (-2,-1) {};
                \node[hg_node] (v6) at (-2,1) {};
                
                % \node at (-1.2, 1.3) {\small $e_1$}; 
                % \node at (1.2, 1.3) {\small $e_2$}; 
                \node at (0, 3.23) {\small $v_1$};
                \node at (2, 1.5) {\small $v_2$};
                \node at (2, -1.5) {\small $v_3$};
                \node at (0, -3.23) {\small $v_4$};
                \node at (-2, -1.5) {\small $v_5$};
                \node at (-2, 1.5) {\small $v_6$};
                \node at (0, 1.5) {\small $e_1$};
                \node at (0, -1.5) {\small $e_3$};
                \node at (2, 0) {\small $e_2$};
                \node at (-2,0) {\small $e_4$};
                \node at (2, -3) {$H$};
                
            \end{tikzpicture}
        };

        % Poset F(X) (Top Right)
        \node (FX) at (7, 3) {
                      \begin{tikzpicture}[scale=0.6]
                % Hyperedge e1 = {v1, v2, v3}

                \fill[hg_edge, fill=red!10, draw=red!50] (0,4) -- (-3.5,0.2) -- (3.5,0.2) -- cycle;
                \fill[hg_edge, fill=red!10, draw=red!50] (0,-4) -- (-3.5,-0.2) -- (3.5,-0.2) -- cycle;
                \fill[hg_edge, fill=red!10, draw=red!50] (-1.2,3.03) -- (-0.3,3.93) -- (3.2,0.7) -- (2.3, -0.2) -- cycle;

                \fill[hg_edge, fill=red!10, draw=red!50] (1.2,-3.03) -- (0.3,-3.93) -- (-3.2,-0.7) -- (-2.3, 0.2) -- cycle;

                \node[hg_node] (v1) at (0,2.73) {};
                \node[hg_node] (v2) at (2,1) {};
                \node[hg_node] (v3) at (2,-1) {};
                \node[hg_node] (v4) at (0,-2.73) {};
                \node[hg_node] (v5) at (-2,-1) {};
                \node[hg_node] (v6) at (-2,1) {};
                
                % \node at (-1.2, 1.3) {\small $e_1$}; 
                % \node at (1.2, 1.3) {\small $e_2$}; 
                \node at (0, 3.23) {\small $v_2'$};
                \node at (1.5, 1) {\small $v_3'$};
                \node at (1.5, -1) {\small $v_4'$};
                \node at (0, -3.23) {\small $v_5'$};
                \node at (-1.5, -1) {\small $v_6'$};
                \node at (-1.5, 1) {\small $v_1'$};
                \node at (0, 1) {\small $e_1'$};
                \node at (0, -1) {\small $e_3'$};
                \node at (0.75, 2) {\small $e_2'$};
                \node at (-0.75, -2) {\small $e_4'$};
                \node at (2, -3) {$H'$};
                
            \end{tikzpicture}
        };

    \end{tikzpicture}
    
    \caption{Two non-isomorphic hypergraphs $H$ and $H'$ that are indistinguishable by HWL test but distinguished by $I$-CatWL test and $S$-CatWL test.}
    \label{fig:notlesspowerfulexamplehypergraphs}
\end{figure}
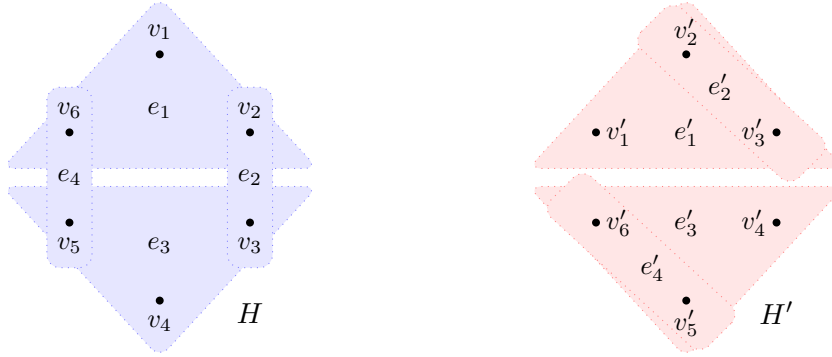

\begin{theorem}\label{SSWLnotlesspowerfulthanGWL}$S$-CatWL is not less powerful than HWL. 
    
\end{theorem}

\begin{proof}We employ the same counterexample from Theorem~\ref{IWLnotlesspowerfulthanGWL} to demonstrate that the hypergraphs $H$ and $H'$ are non-isomorphic under $S$-CatWL. It suffices to show that the colorings $\mathbb{c}^{S,H}_{1}(\{v_2, v_3\}_{e_2})$ and $\mathbb{c}^{S,H'}_{1}(\{v'_2, v'_3\}_{e'_2})$ are distinct, as these correspond to the unique maximal 1-dimensional simplices in their respective structures.

We first compute the aggregated colors from the 4-adjacencies of $\{v_2, v_3\}_{e_2}$ in $H$ at iteration 1. The boundary colors are $\mathbb{c}^{S,H}_{0}(\mathcal{B}(\{v_2, v_3\}_{e_2})) = \{\!\!\{ \mathbb{c}^{S,H}_0(\{v_2\}_{v_2}), \mathbb{c}^{S,H}_0(\{v_3\}_{v_3}) \}\!\!\}= \{\!\!\{ \ast, \ast \}\!\!\}$. The lower-adjacency colors are
\begin{equation}
\begin{split}
\mathbb{c}^{S,H}_{0}(\mathcal{N}_{\downarrow}(\{v_2, v_3\}_{e_2})) &= \{\!\!\{ \mathbb{c}^{S,H}_0(\{v_1, v_2\}_{e_1}), \mathbb{c}^{S,H}_0(\{v_2, v_6\}_{e_1}), \mathbb{c}^{S,H}_0(\{v_3, v_4\}_{e_3}), \mathbb{c}^{S,H}_0(\{v_3, v_5\}_{e_3}) \}\!\!\}\\
&= \{\!\!\{ \ast, \ast, \ast, \ast \}\!\!\}.    
\end{split}
\end{equation}
Both cobounary, upper-adjacency colors are empty.

In contrast, for $\{v'_2, v'_3\}_{e'_2}$ in $H'$, the boundary colors are identical: $\mathbb{c}^{S,H'}_{0}(\mathcal{B}(\{v'_2, v'_3\}_{e'_2})) = \{\!\!\{ \ast, \ast \}\!\!\}$. However, colors of lower-adjacency contain only three elements:

\begin{equation}
    \begin{split}
        \mathbb{c}^{S,H'}_{0}(\mathcal{N}_{\downarrow}(\{v'_2, v'_3\}_{e'_2})) &= \{\!\!\{ \mathbb{c}^{S,H'}_0(\{v'_1, v'_2\}_{e'_1}), \mathbb{c}^{S,H'}_0(\{v'_1, v'_3\}_{e'_1}), \mathbb{c}^{S,H'}_0(\{v'_2, v'_3\}_{e'_1}) \}\!\!\} \\
        &= \{\!\!\{ \ast, \ast, \ast \}\!\!\}.
    \end{split}
\end{equation}
The colors of co-boundary and up-adjacency remain empty.

Consequently, the update rule produces distinct hash values. Specifically, observing that the lower-adjacency inputs differ (4 elements vs 3 elements), we have:
\begin{equation}
\mathrm{HASH}(\dots, \{\!\!\{\ast, \ast, \ast, \ast \}\!\!\}, \dots) \neq \mathrm{HASH}(\dots, \{\!\!\{\ast,\ast,\ast \}\!\!\}, \dots),
\end{equation}
which implies $\mathbb{c}^{S,H}_{1}(\{v_2, v_3\}_{e_2}) \neq \mathbb{c}^{S,H'}_{1}(\{v'_2, v'_3\}_{e'_2})$. By the injectivity of $\mathrm{HASH}$, this confirms that $S$-CatWL successfully distinguishes $H$ and $H'$, whereas HWL does not.
\end{proof}

\section{Equivalence of $F$-CatWL and $F$-CatMPN}
\label{app:equivalenceofCatWLCatMPN}
In this section, we compare the expressive power of $F$-CatWL and $F$-CatMPN.

\begin{theorem}Let $F: \mathcal{C} \to \mathbf{Poset}$ be a functor. The $F$-CatWL is at least as powerful as the $F$-CatMPN. If the update and aggregation functions are injective and the network has sufficient depths and widths, the $F$-CatMPN has the same expressive power as $F$-CatWL.
\end{theorem}

\begin{proof}
Let $\mathbb{c}^{F,X}_k$ denote the $F$-CatWL coloring at iteration $k$, and let $h^{F,X}_k$ denote the $F$-CatMPN node features at layer $k$. We assume the MPN stabilizes or has sufficient layers such that $h^{F,X}_{k} = h^{F,X}_L$ for all $k > L$. We show equivalence by proving $\mathbb{c}^{F,X}_k \sqsubseteq h^{F,X}_k$ and $h^{F,X}_k \sqsubseteq \mathbb{c}^{F,X}_k$ via induction on $k$.

\begin{itemize}[leftmargin=*,topsep=0pt,itemsep=0.5ex]
    \item \textbf{$\mathbb{c}^{F,X}_k \sqsubseteq h^{F,X}_k$:} The base case $k=0$ is trivial. Assume the hypothesis holds for $0, \dots, t$. Suppose $\mathbb{c}^{F,X}_{t+1}(\sigma) = \mathbb{c}^{F,X'}_{t+1}(\sigma')$. The injectivity of the $\mathrm{HASH}$ function implies that the color multisets for $\sigma$ and all four adjacency types ($\mathcal{B}, \mathcal{C}, \mathcal{N}_{\downarrow}, \mathcal{N}_{\uparrow}$) are identical at step $t$. By the induction hypothesis ($h_t \sqsubseteq \mathbb{c}_t$), the corresponding MPN features $h_t$ must also be identical for $\sigma, \sigma'$ and their respective neighborhoods. Since $h_{t+1}$ is a function of these identical inputs, it follows that $h^{F,X}_{t+1}(\sigma) = h^{F,X'}_{t+1}(\sigma')$.

    \item \textbf{$h^{F,X}_k \sqsubseteq \mathbb{c}^{F,X}_k$:} The base case $k=0$ is trivial. Assume the hypothesis holds for $0, \dots, t$. Suppose $h^{F,X}_{t+1}(\sigma) = h^{F,X'}_{t+1}(\sigma')$. We assume the neural aggregation and update functions ($\mathrm{AGG}, \mathrm{UPDATE}$) are injective (which is possible by the Universal Approximation Theorem). This injectivity implies that the input feature multisets for $\sigma, \mathcal{B}, \mathcal{C}, \mathcal{N}_{\downarrow}, \mathcal{N}_{\uparrow}$ at layer $t$ must be identical. By the induction hypothesis ($\mathbb{c}_t \sqsubseteq h_t$), the corresponding WL colors $\mathbb{c}_t$ must also be identical. Since $\mathbb{c}_{t+1}$ is determined by these colors, we have $\mathbb{c}^{F,X}_{t+1}(\sigma) = \mathbb{c}^{F,X'}_{t+1}(\sigma')$.
\end{itemize}
\end{proof}

\section{Experiments Details}
\label{app:experiments}
\subsection{Dataset Details}
\label{app:dataset}

\subsubsection{Dataset Descriptions}
We evaluate our proposed models on six real-world hypergraph classification benchmarks derived from diverse domains, including social media, e-commerce, and collaboration networks. These datasets, originally introduced by \cite{feng2019hypergraph}, are selected to cover a wide range of topological characteristics, from sparse collaboration graphs to dense social ego-networks.

\paragraph{IMDB Datasets.} 
The IMDB datasets consist of hypergraphs constructed from movie collaboration networks, where each hypergraph represents a specific movie and nodes correspond to staff members, including directors and writers. The structural connectivity is determined by co-occurrence relationships, yielding two primary variants: \textbf{IMDB-Dir}, where hyperedges connect directors who have collaborated, and \textbf{IMDB-Wri}, which is constructed based on co-writer relationships. These structures are paired with two distinct classification objectives: identifying the movie's format, such as animation, documentary, and drama, denoted as \textbf{-Form}, or predicting its genre like adventure, crime, and family, denoted as \textbf{-Genre}. Consequently, we evaluate our models on four distinct benchmarks: \texttt{IMDB\_dir\_form}, \texttt{IMDB\_dir\_genre}, \texttt{IMDB\_wri\_form}, and \texttt{IMDB\_wri\_genre}.

\paragraph{Steam-Player.} 
Derived from the Steam gaming platform, this dataset models individual game players as hypergraphs where nodes correspond to owned games. Hyperedges are constructed by linking games that share common user-defined tags like RPG and Action. The associated binary classification task is to determine whether the player prefers single-player or multiplayer games.

\paragraph{Twitter-Friend.} 
The Twitter-Friend dataset focuses on social ego-networks collected from Twitter, where each hypergraph represents the network of a specific user. In this context, nodes denote the friends of the user, and hyperedges connect groups of friends who interact or form circles within that network. The objective is to classify the user's interest based on their tweeting behavior by distinguishing between users who tweeted about "National Dog Day" versus those who tweeted about "Respect Tyler Joseph".

\subsubsection{Statistical Analysis}
\label{app:stat}

We provide a comprehensive statistical analysis of the six datasets from three perspectives: (1) graph-scale statistics, (2) micro-topological properties, and (3) class label distributions.

\paragraph{Graph-Level Statistics (Scale and Density)}
Table \ref{tab:graph_stats} reports the distribution of the number of nodes ($|V|$) and hyperedges ($|E|$) per graph, highlighting the structural diversity across benchmarks. The \texttt{IMDB\_wri} datasets consist of relatively small-scale hypergraphs with fewer than 5 hyperedges on average. In stark contrast, \texttt{steam\_player} and \texttt{twitter\_friend} exhibit extremely dense high-order interactions; for instance, \texttt{steam\_player} shows an average hyperedge-to-node ratio ($|E|/|V|$) of approximately 3.4. Meanwhile, the \texttt{IMDB\_dir} datasets demonstrate significant variance in graph size, with hyperedge counts ranging up to 450, thereby challenging the model's ability to generalize across varying scales.

\begin{table*}[h]
    \centering
    \caption{Graph-level statistics of the datasets. $N_{G}$ denotes the total number of hypergraphs. We report the mean and maximum number of nodes ($|V|$) and hyperedges ($|E|$) per graph.}
    \label{tab:graph_stats}
    \resizebox{\textwidth}{!}{%
    \begin{tabular}{l|c|cc|cc}
        \toprule
        \multirow{2}{*}{\textbf{Dataset}} & \multirow{2}{*}{$N_{G}$} & \multicolumn{2}{c|}{\textbf{Nodes per Graph} ($|V|$)} & \multicolumn{2}{c}{\textbf{Hyperedges per Graph} ($|E|$)} \\
         & & Mean & Max & Mean & Max \\
        \midrule
        \texttt{IMDB\_dir\_form} & 1869 & 15.72 & 264 & 39.30 & 450 \\
        \texttt{IMDB\_dir\_genre} & 3393 & 17.33 & 264 & 36.55 & 450 \\
        \texttt{IMDB\_wri\_form} & 374 & 10.05 & 180 & 3.76 & 23 \\
        \texttt{IMDB\_wri\_genre} & 1172 & 12.82 & 273 & 4.46 & 75 \\
        \texttt{steam\_player} & 2048 & 13.83 & 76 & 46.36 & 140 \\
        \texttt{twitter\_friend} & 1310 & 21.65 & 166 & 84.77 & 603 \\
        \bottomrule
    \end{tabular}%
    }
\end{table*}

\paragraph{Topological Statistics (Degree and Cardinality)}
Table \ref{tab:topo_stats} presents the distribution of node degrees and hyperedge sizes. In terms of node degree, \texttt{twitter\_friend} exhibits the highest mean (19.0) and a heavy tail (max 199), indicating the presence of prominent hub users. Regarding hyperedge cardinality, we observe that while the maximum size in \texttt{twitter\_friend} reaches 64, the 99th percentile remains at 16. This distribution provides empirical justification for our symmetric simplicial lifting strategy with a truncation threshold of 20, as it ensures computational efficiency while preserving the vast majority of local high-order structures.

\begin{table*}[h]
    \centering
    \caption{Distribution of Node Degrees and Hyperedge Sizes. Percentiles reported are the 99th percentile.}
    \label{tab:topo_stats}
    \resizebox{\textwidth}{!}{%
    \begin{tabular}{l|ccc|cc}
        \toprule
        \multirow{2}{*}{\textbf{Dataset}} & \multicolumn{3}{c|}{\textbf{Node Degree Distribution}} & \multicolumn{2}{c}{\textbf{Hyperedge Size Distribution}} \\
         & Mean & Max & 99\%-tile & Max & 99\%-tile \\
        \midrule
        \texttt{IMDB\_dir\_form}   & 8.40 & 68  & 42 & 110 & 12 \\
        \texttt{IMDB\_dir\_genre}  & 7.30 & 77  & 39 & 110 & 14 \\
        \texttt{IMDB\_wri\_form}   & 1.70 & 19  & 6  & 179 & 25 \\
        \texttt{IMDB\_wri\_genre}  & 1.70 & 18  & 7  & 266 & 36 \\
        \texttt{steam\_player}     & 15.20 & 20 & 20 & 71  & 15 \\
        \texttt{twitter\_friend}   & 19.00 & 199& 89 & 64  & 16 \\
        \bottomrule
    \end{tabular}%
    }
\end{table*}

\paragraph{Class Label Distribution}
Table \ref{tab:label_stats} details the class balance for each task. The \texttt{IMDB} variants exhibit significant class imbalance—for instance, one class in \texttt{IMDB\_wri\_genre} constitutes only 2.6\% of the data—whereas \texttt{steam\_player} and \texttt{twitter\_friend} are relatively balanced.

\begin{table*}[h]
    \centering
    \caption{Class label distribution. $C$ denotes the number of classes.}
    \label{tab:label_stats}
    \resizebox{\textwidth}{!}{%
    \begin{tabular}{l|c|l}
        \toprule
        \textbf{Dataset} & $\mathbf{C}$ & \textbf{Distribution (Class Index: Ratio)} \\
        \midrule
        \texttt{IMDB\_dir\_form} & 3 & 0: 42.7\%, 1: 36.8\%, 2: 20.5\% \\
        \texttt{IMDB\_dir\_genre} & 3 & 0: 32.6\%, 1: 16.1\%, 2: 51.3\% \\
        \texttt{IMDB\_wri\_form} & 4 & 0: 45.7\%, 1: 20.9\%, 2: 27.3\%, 3: 6.2\% \\
        \texttt{IMDB\_wri\_genre} & 6 & 0: 24.3\%, 1: 21.9\%, 2: 26.2\%, 3: 18.3\%, 4: 2.6\%, 5: 6.7\% \\
        \texttt{steam\_player} & 2 & 0: 50.0\%, 1: 50.0\% \\
        \texttt{twitter\_friend} & 2 & 0: 59.2\%, 1: 40.8\% \\
        \bottomrule
    \end{tabular}%
    }
\end{table*}

\subsection{Implementation Details.}
\label{app:implementation}
All models were implemented using PyTorch and trained via the Adam optimizer. To ensure a fair comparison, we performed a grid search over the following hyperparameter spaces for all baselines and our proposed models: batch size $\{4, 32, 64\}$, hidden dimension $\{32, 64\}$, number of layers $\{2, 3\}$, and initial learning rate $\{10^{-2}, 5 \times 10^{-3}, 10^{-3}\}$. The dropout rate was fixed at 0.5 across all experiments.

For the training schedule, we employed ReduceLROnPlateau, a mechanism that dynamically decays the learning rate when the validation performance saturates, thereby facilitating fine-grained convergence. Specifically, the learning rate was reduced by a factor of 0.5 if the validation loss did not improve for 20 consecutive epochs, with a minimum threshold of $1 \times 10^{-5}$. 

Regarding the Symmetric Simplicial model (S-HIN), we applied a truncation strategy for the lifting process: hyperedges with a cardinality exceeding 20 were not lifted into high-order simplices to mitigate computational complexity. This design choice is empirically supported by our topological analysis in \Cref{tab:topo_stats}. For dense social and e-commerce benchmarks such as \texttt{steam\_player} and \texttt{twitter\_friend}, the 99th percentile of hyperedge sizes is at most 16. Thus, a truncation threshold of 20 preserves the vast majority of higher-order interactions on these challenging datasets, while only discarding a small number of extremely large hyperedges on the IMDB variants.

All experiments were conducted on a server equipped with an Intel Xeon Gold 5220R CPU @ 2.20GHz, 503GB of RAM, and an NVIDIA GeForce RTX 3090 GPU.

% \paragraph{Implementation Details} 
% All models were implemented using PyTorch and optimized via Adam. To ensure a rigorous comparison, we performed a grid search for key hyperparameters, such as learning rate and network depth, for each model. Training convergence was facilitated by a dynamic scheduler that decays the learning rate upon validation loss saturation. Notably, for the symmetric simplicial complex functor, we adopted a truncation strategy where hyperedges exceeding a cardinality of 20 were excluded from the lifting process. This design choice is empirically supported in the topological analysis provided in Appendix~\ref{app:stat}, which indicates that this threshold accommodates the substantial majority of hyperedges across the datasets, thereby establishing a practical balance between high-order expressivity and computational tractability. Comprehensive experimental configurations are detailed in Appendix~\ref{app:implementation}.

\subsection{Computational Complexity of Lifting Functors}
\label{app:complexity}

In this section, we analyze the computational complexity required to construct the graded posets defined by our two primary functors: the \textbf{Incidence Poset Functor} ($I$) and the \textbf{Symmetric Simplicial Complex Functor} ($S$). We focus on the preprocessing phase, where the raw hypergraph structure $H$ is transformed into its poset representations.

Let $H = (\mathcal{V}, \mathcal{E})$ be a hypergraph with $N = |\mathcal{V}|$ nodes and $M = |\mathcal{E}|$ hyperedges. We denote the cardinality of a hyperedge $e$ as $c_e = |e|$ and the degree of a node $v$ as $d_v = \text{deg}(v)$.

\paragraph{Incidence Poset $I(H)$}
The functor $I$ maps a hypergraph to a poset $I(H)=(V \coprod E, \prec)$ representing the incidence structure. The construction primarily involves establishing the relational structure required for message passing. The base partial order $v \prec e$ (for $v \in e$) necessitates iterating over all node-hyperedge pairs, scaling linearly with the total size of the hypergraph, $\mathcal{O}(\sum_{e \in \mathcal{E}} c_e)$. Furthermore, to facilitate information flow between peer elements, we construct auxiliary adjacencies. The clique expansion, which connects nodes within a hyperedge $e$, requires $\mathcal{O}(c_e^2)$ operations. Similarly, the line graph expansion, connecting hyperedges sharing a node $v$, scales with $\mathcal{O}(d_v^2)$. Consequently, the total complexity is given by $\mathcal{O}\left( \sum_{e \in \mathcal{E}} c_e^2 + \sum_{v \in \mathcal{V}} d_v^2 \right)$. While this complexity is quadratic with respect to hyperedge cardinalities and node degrees, it remains tractable for sparse real-world hypergraphs, allowing $I(H)$ to be constructed without structural truncation.

\paragraph{Symmetric Simplicial Complex Poset $S(H)$}
The functor $S$ maps a hypergraph to a poset $S(H)$ composed of structured simplices. The construction of the $n$-th level set $S(H)_n$ involves a combinatorial enumeration of subsets. Specifically, constructing $S(H)_n$ requires generating $\binom{c_e}{n+1}$ simplices for each hyperedge. For a standard configuration up to 2-simplices (triangles, $n=2$), the complexity is cubic, expressed as $\mathcal{T}_{gen} = \mathcal{O}\left( \sum_{e \in \mathcal{E}} c_e^3 \right)$. This polynomial scaling poses significant challenges for datasets containing hyperedges with large cardinality. For instance, in \texttt{IMDB\_wri\_genre}, the maximum hyperedge size reaches 266. Explicitly generating all triangles for such a hyperedge would produce millions of elements from a single interaction, rendering the process computationally prohibitive. To ensure tractability, we implement a \textit{truncated evaluation} of $S$ with a threshold $\tau = 20$, restricting the domain to hyperedges where $c_e \le \tau$. Under this constraint, the complexity is bounded by
$$ \mathcal{T}_{S(H)} \approx \mathcal{O}\left( \sum_{e \in \mathcal{E}, c_e \le \tau} c_e^3 \right) \approx \mathcal{O}\left( M \cdot \tau^3 \right). $$
By treating $\tau$ as a constant, the construction of $S(H)$ becomes linear with respect to the number of hyperedges $M$. This analysis justifies our truncation strategy, ensuring efficient poset construction while preserving the high-order geometry for the vast majority of the data.

% \crefalias{section}{appendix} % uncomment if you are using cleveref

\end{document}